\pdfoutput=1
\documentclass[11pt]{article}

\usepackage{fullpage}
\usepackage{amsmath,amssymb,amsthm,graphicx} 
\usepackage{epsfig}
\usepackage[numbers]{natbib} 
\usepackage{psfrag}
\usepackage{enumerate} 
\usepackage{enumitem} 
\usepackage{setspace}
\usepackage{float} 
\usepackage{color}
\usepackage{array}
\usepackage{pgf,tikz}
\usepackage{mathrsfs}
\usepackage{hyperref}
\usepackage{subcaption}
\usepackage{wrapfig}
\usepackage[ruled,vlined,linesnumbered]{algorithm2e}
\usepackage{tabularx}
\usepackage{makecell}
\usepackage[affil-it]{authblk}
\usepackage{cleveref}
\usepackage{comment}

\usepackage[T1]{fontenc}
\usepackage{lmodern}
\usepackage{caption}
\usepackage{adjustbox}
\usepackage{geometry}

\usepackage[utf8]{inputenc} 
\usepackage[T1]{fontenc}    
\usepackage{hyperref}       
\usepackage{url}            
\usepackage{booktabs}       
\usepackage{amsfonts}       
\usepackage{nicefrac}       
\usepackage{microtype}      

\usetikzlibrary{arrows}
\definecolor{ffqqqq}{rgb}{1.,0.,0.}
\definecolor{xfqqff}{rgb}{0.4980392156862745,0.,1.}

\definecolor{mypink}{rgb}{0.858, 0.188, 0.478}
\definecolor{myred}{rgb}{1, 0, 0}

\makeatletter
\long\def\@makecaption#1#2{
        \vskip 0.8ex
        \setbox\@tempboxa\hbox{\small {\bf #1:} #2}
        \dimen0=\hsize
        \advance\dimen0 by 0cm
        \ifdim \wd\@tempboxa >\dimen0
                \hbox to \hsize{
                        \parindent 0em
                        \hfil 
                        \parbox{\dimen0}{\def\baselinestretch{0.96}\small
                                {\bf #1.} #2
                                } 
                        \hfil}
        \else \hbox to \hsize{\hfil \box\@tempboxa \hfil}
        \fi
        \vspace{0.4cm}
        }
\makeatother

\usepackage{macros}

\usepackage{amsmath,amsfonts,bm}

\def\1{\bm{1}}

\def\vtheta{{\bm{\beta}}}

\DeclareMathAlphabet{\mathsfit}{\encodingdefault}{\sfdefault}{m}{sl}
\SetMathAlphabet{\mathsfit}{bold}{\encodingdefault}{\sfdefault}{bx}{n}

\newcommand{\R}{\mathbb{R}}

\DeclareMathOperator*{\argmin}{arg\,min}

\DeclareMathOperator{\sign}{sign}

\usepackage{etoolbox}
\makeatletter
\patchcmd{\@algocf@start}
  {-1.5em}
  {0pt}
  {}{}
\makeatother

\makeatletter
\renewcommand{\paragraph}{%
  \@startsection{paragraph}{4}%
  {\z@}{0.25ex \@plus 1ex \@minus .2ex}{-1em}%
  {\normalfont\normalsize\bfseries}%
}
\makeatother

\newif\ifarxiv
\arxivtrue

\renewcommand\citet{\citep}

\makeatletter
\appto\@floatboxreset{%
  \ifx\@captype\andy@table
    \sffamily
  \fi
}
\def\andy@table{table}
\makeatother

\usepackage{caption} 
    \newif\ifshownotes
    \shownotestrue

\title{Can semi-supervised learning use all the data effectively?\\A lower bound perspective}

\author{Alexandru \textcommabelow{T}ifrea\thanks{Equal contribution. Presented at the 37th Conference on Neural Information Processing Systems (NeurIPS 2023).} $^1$, Gizem Y\"uce$^{*1}$, Amartya Sanyal$^2$ and Fanny Yang$^1$
}

\affil{$^1$ETH Zurich,
$^2$Max Planck Institute for Intelligent Systems, T\"ubingen,
}

\begin{document}
\hypersetup{pageanchor=false}

\maketitle

\begin{abstract}

Prior works
have shown that semi-supervised learning algorithms can leverage unlabeled data to improve 
over the labeled sample complexity of supervised learning (SL) algorithms. However, existing theoretical analyses focus on regimes 
where the unlabeled data is sufficient to learn a good decision boundary using unsupervised learning (UL) alone. This begs the question: Can SSL algorithms simultaneously improve upon both UL \emph{and} SL? To this end, we derive a tight lower bound for 2-Gaussian mixture models that explicitly depends on the labeled and the unlabeled dataset size as well as the signal-to-noise ratio of the mixture distribution. Surprisingly, our result implies that no SSL algorithm can improve upon the minimax-optimal statistical error rates of SL or UL algorithms for these distributions. 
Nevertheless, we show empirically on real-world data that SSL algorithms can still outperform UL and SL methods. 
Therefore, our work suggests that, while proving performance gains
for SSL algorithms is possible, it requires 
careful tracking of constants.

\end{abstract}

\section{Introduction}
\label{sec:intro}
Semi-Supervised Learning (SSL) has recently garnered significant
attention, often surpassing traditional supervised learning (SL)
methods in practical applications \citep{balestriero2023cookbook,
10.5555/3495724.3497589SIMCLR, grill2020bootstrapbyol}. Within this
framework, the learning algorithm leverages both labeled and unlabeled
data sampled from the same distribution. Numerous empirical
studies suggest that SSL can effectively harness the information
from both datasets, outperforming both SL and unsupervised learning
(UL) approaches~\citep{goyal2019scaling, zoph2020rethinking,
el2021large, lucas2022barely}. This observation prompts the question:
how fundamental is the improvement of SSL over SL and UL?
\footnote{By error of UL we mean the prediction error up to sign. We formalize this paradigm of using UL first and then identifying the correct sign as \ulplus in \Cref{sec:setting-alg}.}

Prior theoretical results do not provide a consensus on this question. They consider 
different sample size regimes and degrees of 
``compatibility'' \citep{balcanblumcomp} between the marginal and the conditional distributions.
One line of work demonstrates that with enough unlabeled data, 
SSL is capable of lowering the labeled sample complexity compared to SL, for mixture models~\citep{ratsaby,pmlr-v151-frei22a} or more general ``clusterable'' distributions~\citep{rigollet}.
These scenarios have in common that the unlabeled data possesses a significant amount of information about the conditional distribution and SSL only performs as well as UL (up to permutation). 
Another line of work challenges the above literature by identifying distributions where SSL achieves the same error rate as SL \citep{BenDavid2008DoesUD, Tolstikhin2016MinimaxLB, Gpfert2019WhenCU}. These analyses consider scenarios where even oracle knowledge of the marginal distribution fails to improve upon the error rates of SL algorithms since the marginal does not carry any information about the labeling i.e.\ the conditional distribution. In these situations, SSL can only improve upon UL, but not SL.

In summary, the previous bounds do not provide a conclusive answer on the benefits of SSL over SL and UL. 
Therefore, in this paper, we provide a first answer to the following question:
\begin{center}

  \emph{Can semi-supervised classification algorithms simultaneously
  improve \\over the minimax rates of both SL and UL?}

\end{center}

In~\Cref{sec:setting,sec:theory_main}, we study this question in
the context of linear classification for symmetric 2-Gaussian mixture models (GMMs) -- a standard setting used in theoretical analyses of classification methods~\citep{ratsaby, pmlr-v151-frei22a,NIPS2013_456ac9b0, pmlr-v54-li17a, Wu2019RandomlyIE}.
We derive minimax error rates for semi-supervised learning that specifically depend on the ``regime'', characterized by three quantities: the amount of available unlabeled data~\(\nunl\), amount of available labeled data~\(\nlab\), and the inherent signal-to-noise ratio~(SNR) 
that quantifies the amount of information the marginal input
distribution has about the conditional label distribution (see~\Cref{sec:theory_main} for more details). An SNR-dependent minimax rate allows us to analyze the entire spectrum of problem difficulties for 2-GMMs. By contrasting the SSL minimax rates with established minimax rates for SL and UL, we find that in no regime can SSL surpass the statistical rates of both SL and UL simultaneously. 
In conclusion, a minimax-optimal SSL algorithm for 2-GMM is one that adeptly switches between SL and UL algorithms depending on the regime, and hence, it never uses the available data fully.

In~\Cref{sec:better_ssl}, we demonstrate how statistical rates may not offer a complete picture for explaining the practical benefits of SSL algorithms. Several prevalent SSL algorithms, such as self-training \citep{10.3115/981658.981684YAROWSKY, blum98}, are more sophisticated than UL approaches and use labeled data not only for determining the sign but also for learning the decision boundary. 
We show that an SSL ensembling method and self-training can indeed improve upon the best of SL and UL algorithms, in proof-of-concept experiments for linear classification tasks on both synthetic and real-world data. Since the improvements cannot be captured by statistical rates, our results highlight the significance of the constant factors in future theoretical work that analyzes the advantage of SSL algorithms.

\section{Problem setting and background}
\label{sec:setting}

Before providing our main results, in this section, we first define the
problem setting for our theoretical analysis, the evaluation metrics and the types of learning
algorithms that we compare. We then describe how we compare the rates between SSL and supervised and unsupervised learning

\subsection{Linear classification for 2-GMM data}
\label{sec:setting_gmm}

\paragraph{Data distribution.} We consider linear binary classification
problems where the data is drawn from a Gaussian mixture model (GMM)
consisting of two identical spherical Gaussians with identity
covariance and uniform mixing weights.  The means of the two
components \(\thetastar, -\thetastar\) are symmetric with respect to
the origin but can have arbitrary non-zero norm. We denote this family
of joint distributions as $\Pgmm := \{\jointsnr : \thetastar \in \RR^d\}$
and note that it can be factorized as 
\begin{equation}
  P_Y = \unifdist\{-1, 1\} \text{ and }  P^{\thetastar}_{X|Y} = \gauss(Y \thetastar, I_d). \label{eq:symm2GMM}
\end{equation}
This family of distributions has often been considered in the context
of analysing both SSL \citep{ratsaby, pmlr-v151-frei22a} and SL/UL
\citep{NIPS2013_456ac9b0, pmlr-v54-li17a, Wu2019RandomlyIE}
algorithms.  For $\snr \in (0, \infty)$, we denote by $\Pgmmsnr\subset\Pgmm$ and $\thetaset^{(\snr)} \subset \RR^d $
the set of distributions \(\jointsnr\)  and the set of parameters with
$\norm{\thetastar}=\snr$. With error rates that depend on $\snr$, we can explicitly capture the intrinsic information about the conditional $Y|X$ contained in the marginal distribution.
We consider algorithms $\alg$ that take as
input a labeled dataset $\Dlab\sim\left(\jointsnr\right)^{\nlab}$ of
size $\nlab$, an unlabeled dataset $\Dunl\sim
\left(\marginalsnr\right)^{\nunl}$ of size $\nunl$, or both, and
output an estimator $\thetahat=\alg\left(\Dlab, \Dunl\right)\in\RR^d$. The
estimator is used to predict the label of a test point $x$ as $\yhat =
\sign\left(\langle \thetahat, x \rangle\right)$.

\paragraph{Evaluation metrics} For a given parameter $\vtheta$ we define two natural
error metrics for this class of problems which are the
%
\begin{align}
  &\textbf{Prediction error: } \prederr\left(\thetavec, \thetastar\right):=
  \jointsnr\left(\sign\left(\langle \thetavec, X
  \rangle\right) \neq Y\right), \quad \text{and} \notag\\
  &\textbf{Estimation error: } \estimerr\left(\thetavec, \thetastar\right):= \norm{\thetavec - \thetastar}_2.\label{eq:est-error}
\end{align}
When the true $\thetastar$ is clear from the context, we drop the second argument for simplicity. 
For linear classification and a
2-GMM distribution from the family $\Pgmm$, a low estimation error
implies not only good predictive performance but also good
calibration under a logistic model \citep{Plat1999}.
In our discussions, we focus on the prediction error, but include a minimax rate for the estimation error for completeness. In particular, we bound the excess risk of $\thetavec$, defined as the distance between the risk of $\thetavec$ and the Bayes optimal risk
\begin{align*}\label{eq:excess-pred-err}
  \textbf{Excess prediction error: } \excess\left(\thetavec, \thetastar\right):=\prederr\left(\thetavec, \thetastar\right)- \inf_{\thetavec}\prederr\left(\thetavec,
  \thetastar\right),
\end{align*}
where $\inf_{\thetavec}\prederr\left(\thetavec,
  \thetastar\right)$ is achieved at $\thetastar$ but can be non-zero.



For a certain class of algorithms (i.e.\ SL, UL+, SSL), 
we study the worst-case expected error
over parameters $\thetastar \in \thetaset$. This
quantity indicates the limits of what is
achievable with the algorithm class. For instance, the
minimax optimal expected excess error of the algorithm class $\alg$ over parameters in \(\thetaset\) takes the
form:
\begin{align}
  \textbf{Minimax excess error: } \minimax\br{\nlab, \nunl, \thetaset} :=
  \inf_{\alg} \sup_{\thetastar \in \thetaset} \EE
  \left[\excess\left(\alg(\Dlab, \Dunl), \thetastar\right)\right].
\end{align}


\vspace{-0.3cm}
\subsection{Minimax optimal rates of supervised and unsupervised learning}
\label{sec:setting-alg}

We distinguish between three kinds of learning paradigms that can be used in the SSL setting to learn a decision boundary $\thetahat$ but are designed to leverage the available data differently.
For simplification, our discussion is tailored towards learning \(\Pgmm\), though the ideas hold more generally.

\paragraph{1) Semi-supervised learning (SSL)} SSL algorithms,
denoted as \(\algssl\), can utilize both labeled data \(\Dlab\) and unlabeled
samples \(\Dunl\) to learn the decision boundary and to produce an estimator
$\thetassl=\algssl\br{\Dlab, \Dunl}$. The promise of SSL is
that by combining labeled and unlabeled data, SSL can reduce both the
labeled and unlabeled sample complexities compared to algorithms
solely dependent on either dataset.

\begin{table}\centering
  \begin{tabular}{ccc}
    \toprule
    Learning paradigm & Excess risk rate & Estimation error rate\\
    \midrule
    SL & \( \Theta\left(e^{-\nicefrac{\snr^2}{2}}\frac{d}{\snr\nlab}\right) \) & \(  \Theta\left(\sqrt{\frac{d}{\nlab}}\right) \)\\\midrule
    
    UL(+) & \( \tilde{\Theta}\left(e^{-\nicefrac{\snr^2}{2}}\frac{d}{\snr^3\nunl}\right) \) & \( \Theta\left(\sqrt{\frac{d}{\snr^2\nunl}}\right) \) \\
    \bottomrule
  \end{tabular}
  \caption{Known minimax rates of SL and UL for learning 2-GMMs with \( \snr \in (0,1] \) \citep{Wu2019RandomlyIE, pmlr-v54-li17a}. The minimax rates for UL are up to choosing the correct sign. This rate is the same as for \(\ulplus\) if \(\nlab\geq \Omega(\log \nunl )\). For \( \snr > 1 \), the estimation error rate of UL is the same as SL. Hence, there is no trade-off one can study between SSL and UL or SL.
  }
  \label{table:known_rates}
  \vspace{-0.5cm}
\end{table}

\paragraph{2) Supervised learning (SL)} SL algorithms,
represented by $\algsl$, can only use the labeled dataset $\Dlab$ to yield an estimator $\thetasl = \algsl\br{\Dlab,\emptyset}$. 
The minimax rates of SL for distributions from $\Pgmmsnr$ (see~\Cref{table:known_rates}) are achieved by the mean estimator \( \thetasl =  \frac{1}{\nlab} \sum_{i=1}^{\nlab} Y_i X_i \), for both excess risk and estimation error.

\paragraph{3) Unsupervised learning (UL)} Traditionally, UL algorithms are tailored to learning the generative model for marginal distributions. For \(\Pgmm\) the marginal is governed by $\thetastar$ and UL algorithms
output a set of estimators $\{\thetaul,
-\thetaul\}=\algul\br{\emptyset, \Dunl}$ one of which is
guaranteed to be close to the true $\thetastar$. 
To evaluate prediction performance, we define the minimax rate of UL algorithms as the minimax rate for the closer~(to the true $\thetastar$) of the two estimators. This minimax rate of UL algorithms over $\Pgmmsnr$ is known for both the excess risk and the estimation error
\citep{pmlr-v54-li17a, Wu2019RandomlyIE} (see~\Cref{table:known_rates}). These rates are achieved by the unsupervised estimator \( \thetaul = \sqrt{(\hat{\lambda}-1)_+}\hat{v}\), where $(\hat{\lambda}, \hat{v})$ is the leading eigenpair of the sample
covariance matrix $\hat{\Sigma} = \frac{1}{\nunl} \sum_{j=0}^{\nunl} X_j
X_j^T$ and we use the notation $(x)_+:=\max(0, x)$.


\subsection{A "wasteful" type of SSL algorithm}

\begin{wrapfigure}{r}{0.45\textwidth}
\vspace{-0.65cm}
\begin{minipage}[t]{0.45\textwidth}
      \setlength{\algomargin}{0.4cm}
    \begin{algorithm}[H]
      \DontPrintSemicolon
      \small

   \SetKwInOut{KwInput}{Input}

   \KwInput{$\Dlab$, $\Dunl$}


   $\{\thetaul, -\thetaul\} \leftarrow \algul(\Dunl)$  \;

   $\thetaulp \leftarrow$ Select from $\{\thetaul, -\thetaul\}$ using $\Dlab$

   \KwRet \(\thetaulp\) \;

   \caption{\ulplus algorithms \(\algulp\)}
   \label{algo:ulplus}
  \end{algorithm}
\end{minipage}
\vspace{-0.8cm}
\end{wrapfigure}  

Several SSL algorithms used in practice (e.g.\ SimCLR~\citep{chen2020simple},
SwAV~\citep{caron2020unsupervised}) follow a two-stage procedure: 
i) determine decision boundaries using only unlabeled data; and ii) label decision regions using only labeled data. We
refer to this class of two-stage algorithms as \textbf{\ulplus} and denote them by \(\algulp\). Early analyses of semi-supervised learning focus, in fact, on algorithms that fit the description of~\ulplus \citep{ratsaby,rigollet}.

For linear binary classification, the two-stage framework is depicted in~\Cref{algo:ulplus}. 
The following proposition upper bounds the excess risk and estimation error of the \ulplus estimator given by
\begin{equation}
\label{eq:thetaulp}
    \thetaulp = \sign\left(\thetasl^\top  \thetaul\right)  \thetaul, 
\end{equation}  
where $\thetasl$ and $\thetaul$ are the rate-optimal SL and UL estimators described in~\Cref{sec:setting-alg}. 
\begin{proposition}[Upper bounds for $\thetaulp$ ]\label{prop:ul+main}
Let $\thetaulp$ be the estimator defined in~\Cref{eq:thetaulp}. For any $\snr >0$ the following holds when $\nunl \gtrsim \max\{d, \frac{d}{\snr^4},\frac{\log(\nunl)}{\min\{\snr^2, \snr^4\}}  , \frac{d\log(d\nunl)}{\snr^6}  \}$ and $d \ge 2$:
\begin{align}
     &\EE \left[ \excess\left(\thetaulp, \thetastar\right)\right] \lesssim e^{-\frac{1}{2}s^2}
         \frac{d}{\snr^3 \nunl} + e^{-\frac{1}{2}  \nlab\snr^2  \left(1- \frac{c_0 }{\min\{\snr, \snr^2\}}\sqrt{\frac{d \log(\nunl)}{ \nunl}}\right)^2 },   \text{ and for } s \in (0,1]  \notag \\
    &\EE \left[  \estimerr\left(\thetaulp, \thetastar\right) \right] \lesssim \sqrt{\frac{d}{s^2 \nunl}} + s  e^{-\frac{1}{2} \nlab \snr^2 \left(1- \frac{c_0 }{\snr^2} \sqrt{\frac{d \log(\nunl)}{\snr^2 \nunl}}\right)^2}. \hspace{10pt} \notag 
\end{align}
\end{proposition}
\Cref{appendix:ulplus} contains the complete statement of the proposition including logarithmic factors and the proof. 
Note that for 
large enough $\nlab$ the first term in the upper bound dominates, thereby resulting in the same rate as the minimax rate of UL up to choosing the correct sign.
We remark that, while~\Cref{algo:ulplus} defines~\ulplus algorithms as only using the unlabeled dataset \(\Dunl\) for the unsupervised
learning step, one can also use the labeled dataset \(\Dlab\) without labels in that step. However, typically, in practice, \ulplus style algorithms (e.g.\ SimCLR, SwAV) do not use the labeled data in this way, as they operate in a regime where  unlabeled data is far more numerous than labeled data. Thus, we analyze~\Cref{algo:ulplus} in this work.

\paragraph{Why UL+ algorithms are ``wasteful''} As
indicated in~\Cref{algo:ulplus},~\ulplus type algorithms follow a precise
structure where labeled data is used solely to select from the set of
estimators output by a UL algorithm. Intuitively, such algorithms do not take full advantage of the labeled data as it is not used to refine the decision boundary. 
In prior empirical and theoretical studies, this inefficiency has not been a problem since they have focused on the regime $n_u =\omega(n_l)$, where unlabeled data is often orders of magnitude more abundant than labeled data. When  $\nunl = \Theta(\nlab)$ or even $\nunl \ll \nlab$, however, \Cref{prop:ul+main} shows how this two-stage approach of \ulplus can become strikingly ineffective.
For simplicity, consider the extreme scenario where $\nunl$ is
finite, but $\nlab\to\infty$. The error of a~\ulplus algorithm will,
at best, mirror the error of a UL algorithm with the correct sign
(e.g.\ $\Theta\br{\nicefrac{d}{\nunl}}$ for the excess risk). Thus, despite
using both labeled and unlabeled data, ~\ulplus algorithms bear a
close resemblance to UL algorithms that only use unlabeled data.

\subsection{Brief overview of prior error bounds for SSL}
\label{sec:prior-upper} 

In this section, we discuss prior theoretical works that aim to show benefits and limitations of SSL.

\paragraph{Upper bounds} 
There are numerous known upper bounds on the excess risk of SSL algorithms for $\Pgmm$ distributions. 
However, despite showing better dependence on the labeled set size, these earlier bounds primarily match the UL+  rates~\citep{ratsaby, rigollet} or exhibit slower rates than UL+~\citep{pmlr-v151-frei22a}. 
That is, we cannot conclude from existing results that SSL algorithms can consistently outperform \emph{both} SL and UL+, which is the question we aim to address in this paper.


\paragraph{Lower bounds} In contrast to the upper bounds that aim to demonstrate benefits of SSL, three distinct minimax lower bounds for SSL have
been proposed to show the limitations of SSL. Each  proves, in different settings, that there exists a distribution $\joint$
where SSL cannot outperform the SL minimax rate.
\cite{BenDavid2008DoesUD} substantiates this claim for learning
thresholds from univariate data sourced from a uniform distribution on
\(\bs{0,1}\).~\cite{Gpfert2019WhenCU} expands upon this by considering
arbitrary marginal distributions $\marginal$ and a ``rich'' set of
realizable labeling functions, such that no volume of unlabeled data
can differentiate between possible hypotheses. Lastly,
\cite{Tolstikhin2016MinimaxLB} sets a lower bound for scenarios with
no implied association between the labeling function and the marginal
distribution, a condition recognized as being unfavorable for SSL
improvements~\citep{scholkopf2012}.
Each of the aforementioned results contends that a particular worst-case distribution $\joint$ exists, where the labeled sample complexity for SSL matches that of SL, even with limitless unlabeled data. 

To summarize, the upper bounds show that SSL improves upon SL in settings where \(\nunl\) is significantly larger than $\nlab$. In fact, for such large unlabeled sample size even \(\ulplus\) can achieve small error. On the other hand, the lower bounds prove the futility of SSL for distributions where even infinite unlabeled data cannot help i.e.\ \(\marginal\) does not contain sufficient information about the conditional distribution. However these works fail to answer the question: does there exist a relatively moderate \(\nunl\) regime where SSL is better than both SL and UL+?

In the family of~\(\Pgmm\) distributions, the above lower bounds translate to the hard setting where \(\nunl\ll\nicefrac{1}{\snr}\). We now aim to prove a minimax lower bound for a fixed difficulty~\(\snr>0\), that allows us to answer the above question in different regimes in terms of \(\nlab,\nunl\).


\section{Minimax rates for SSL}
\label{sec:theory_main}
In this section we provide tight minimax lower bounds for SSL
algorithms and 2-GMM distributions in $\Pgmmsnr$. Our results
indicate that it is, in fact, not possible for SSL algorithms to
simultaneously achieve faster minimax rates than both SL and UL+.

\subsection{Minimax rate}\label{sec:excesslb}

We begin by introducing tight lower bounds on the excess risk~\eqref{eq:excess-pred-err} of a linear estimator obtained using both labeled and unlabeled data. In addition, we also present tight lower bounds on the estimation error~\eqref{eq:est-error} of the means of class-conditional distributions obtained similarly using labeled and unlabelled data. This is especially relevant when addressing the linear classification of
symmetric and spherical GMMs. In this setting, a reduced estimation error points to not only a low excess risk but also suggests a small calibration
error under the assumption of a logistic noise
model~\citep{Plat1999}. Both of these results are presented in~\Cref{thm:excess_main}. We present short proof sketches here and relegate the formal conditions required by the theorem to hold as well as the full proofs to~\Cref{appendix:excess,appendix:estimation}~(for the estimation error and excess risk, respectively).


\begin{theorem}[SSL Minimax Rate for Excess Risk and Estimation Error]
\label{thm:excess_main}
\label{thm:param_est_main}
Assume the conditions in~\Cref{prop:ul+main} and additionally let $\nlab > O(\frac{\log \nunl}{\snr^2})$. Then for any $\snr > 0$, 
we have
\begin{align*}
  &\inf_{\algssl} \sup _{\norm{\thetastar}=\snr} \EE\bs{ \errssl} = \Theta\left( e^{-\nicefrac{\snr^2}{2}} \min \bc{s, 
\frac{d}{\snr \nlab + \snr^3 \nunl}} \right),\text{ and for $s\in(0, 1]$}\\
    &\inf _{\algssl} \sup _{\norm{\thetastar}=\snr}
    \EE\bs{\estimerr(\algssl(\Dlab, \Dunl), \thetastar)} =
  \tilde{\Theta} \left(\min \bc{ s, \sqrt{\frac{d}{ \nlab+ \snr^2 \nunl}}}\right),
\end{align*}
where the infimum is over all the possible SSL algorithms that have access to both unlabeled and labeled data and the expectation is over $\Dlab \sim
    \br{\jointsnr}^{\nlab}$ and $\Dunl \sim
    \br{\marginalsnr}^{\nunl}$. 
\end{theorem}


The notation $\tilde{\Theta}$ also hides negligible logarithmic factors. In the rest of the section, we focus solely on the bound for excess risk; however, we note that the discussion here transfers to estimation error as well. In~\Cref{sec:analysis}, we discuss the new insights that this bound provides. A direct implication of the result is that
$\minimaxsslgmm =\Theta\left(\min\br{\minimaxslgmm, \minimaxulpgmm}\right)$,  that is, the minimax rate of SSL is the same as either that of SL or UL+,  depending on the value of $\snr$ and the rate of growth of \(\nunl\) compared to \(\nlab\). 
Therefore, we can conclude that \emph{no SSL algorithm can simultaneously improve the rates of both SL and UL+ for $\vtheta \in \thetaset$.}
We provide a detailed discussion of the rate improvements in~\Cref{sec:analysis}.


\paragraph{Proof sketch} The full proof of the lower bound for excess risk is presented in \Cref{appendix:excess}.
For this proof, we adopt the packing construction in \cite{pmlr-v54-li17a} and apply Fano's method together with techniques introduced in \citet{NIPS2013_456ac9b0}. Since the algorithms have access to both the labeled and unlabeled datasets in the semi-supervised setting, KL divergences between both the marginal and the joint distributions appear in the lower bound after the application of Fano's method, which is the key difference from its SL and UL counterparts. 

\SetKwIF{If}{ElseIf}{Else}{if}{}{else if}{else}{end if}%
\begin{wrapfigure}{r}{0.46\textwidth}
 \vspace{-0.5cm}
\begin{minipage}[b]{0.46\textwidth}
        \setlength{\algomargin}{0.4cm}
      \begin{algorithm}[H]
        \DontPrintSemicolon
        \small

     \SetKwInOut{KwInput}{Input}

     \KwInput{$\Dlab$, $\Dunl$, $\snr$, $\algsl$, $\algulp$ }

     \KwResult{\(\thetassls\)}

     $\thetasl \leftarrow \algsl(\Dlab)$ \;
     
     $\thetaulp \leftarrow \algulp(\Dunl, \Dlab)$ \;

    \If{$s \leq \min \left \{\sqrt{\frac{d}{\nlab}},\left( \frac{d}{\nunl} \right)^{1/4} \right \}$}
        {$\thetassls=0$}
    \ElseIf{$\min \left \{\sqrt{\frac{d}{\nlab}},\left( \frac{d}{\nunl} \right)^{1/4} \right \} < s \leq \sqrt{\frac{\nlab}{\nunl}}$}
        {$\thetassls=\thetasl$}
    \Else
        {$\thetassls=\thetaulp$}
     
     \KwRet \(\thetassls\) \;

     \caption{SSL-S algorithm}
     \label{algo:ssl_switch}
\vspace{-0pt}
    \end{algorithm}
  \end{minipage}
  \vspace{-0.5cm}
\end{wrapfigure}

We then show that the rate can be matched by the \textbf{SSL Switching Algorithm (SSL-S)} in Algorithm~\ref{algo:ssl_switch} -- an oracle algorithm that switches between using a (minimax
optimal) SL or UL+ algorithm based on the values of $\snr, \nlab,$ and
$\nunl$. The upper bound then follows as a corollary from Proposition \ref{prop:ul+main} and the upper bounds for supervised learning. 

For the parameter estimation 
error lower bound, we use Fano's method with the packing construction in 
\cite{Wu2019RandomlyIE}, who have employed this method to
derive lower bounds in the context of unsupervised learning. Similar to the excess risk analysis, the lower bound reveals that the SSL rate is either determined by the
SL rate or the~\ulplus rate depending on $\snr$ and the ratio of the sizes of the labeled and unlabeled samples.  Once again, the minimax error rate is matched by the SSL Switching algorithm presented in~\Cref{algo:ssl_switch}.

\paragraph{Discussion of the details of the theorem}
We note that the SSL-S algorithm chooses to output the trivial estimator $\thetassls=\mathbf{0}$ if the SNR $\snr$ is low. In the low-SNR regime, the trivial estimator $\thetahat=\mathbf{0}$ achieves a small excess risk of
$\excess\br{\mathbf{0}, \thetastar}=\Phi\br{\snr}-\Phi\br{0}\simeq
e^{-\nicefrac{\snr^2}{2}}\snr$, where $\Phi$ is the CDF of the standard normal distribution.

Furthermore, the lower bound for the excess risk in~\Cref{thm:excess_main}, is only tight up to
logarithmic factors. We conjecture that the logarithmic factors are an
artifact of the analysis and can be removed. For instance, it may be
possible to extend results in~\cite{ratsaby} that bound the excess
risk using the estimation error upper bound without incurring
logarithmic factors. However, their results are not directly
applicable here because they are only valid under the assumption that the estimation error is arbitrarily small.

Finally, note that for \( \snr > 1 \), the estimation error rate of UL is the same as SL. Hence, in this regime, there is no trade-off one can study between SSL and UL or SL because the error rate is independent of the percentage of available labels in the data. Therefore, we only focus on the setting where \( \snr \in (0,1] \).
While the \emph{rates} of either SL or UL+ cannot be improved further
using SSL algorithms, it is nonetheless possible to improve the error
by a constant factor, as discussed in Section~\ref{sec:better_ssl}.



\subsection{Comparison of SSL with UL+ and SL in different regimes}
\label{sec:analysis}

To understand whether 
SSL algorithms can fundamentally use 
labeled and unlabeled data more effectively than SL or \ulplus we compare the minimax error rates between the different learning paradigms.
A faster minimax rate for SSL would imply that indeed there exists an SSL algorithm that provably outperforms both SL and \ulplus. 
In this section, we study the \emph{rate improvement} defined as the ratio 
between the minimax rates for SSL and 
SL or \ulplus.

\begin{definition}[SSL rate improvement]
  \label{ssl_helps}

For any set of parameters $\thetaset \subseteq \RR^d $, we define the rate improvements 
of SSL over SL and \ulplus as 
\begin{align}
  \ratel(\nlab, \nunl, \thetaset) &:= \dfrac{\inf_{\algssl} \sup_{\thetastar \in
  \thetaset}\EE \bs{\errsslgen}}{\inf_{\algsl}\sup_{\thetastar \in
  \thetaset}\EE\bs{\errslgen}}, \label{eq:cond1}\\
  \rateu(\nlab, \nunl, \thetaset) &:= \dfrac{\inf_{\algssl} \sup_{\thetastar \in
  \thetaset} \EE\bs{\errsslgen}}{\inf_{\algulp}\sup_{\thetastar \in
  \thetaset}\EE\bs{\errulgen}}, \label{eq:cond2}
\end{align}
where the expectations are over $\Dlab \sim {(\joint^{\, \thetastar})}^{\nlab}$ and $\Dunl
\sim {(\marginal^{\, \thetastar})}^{\nunl}$.

\end{definition}

A straightforward upper bound for
these rate improvement ratios is $\ratel,\rateu \le 1$, achieved when utilizing an SL
and UL+ algorithm, respectively. SSL demonstrates an enhanced
error rate over SL and UL+, if both $\lim_{\nlab,\nunl \to
\infty}\ratel(\nlab, \nunl, \thetaset)=0$ and $\lim_{\nlab,\nunl \to
\infty}\rateu(\nlab, \nunl, \thetaset)=0$. If 
$\ratel$ or $\rateu$ lies in $(0, 1)$ without converging to zero as
$\nlab,\nunl\to\infty$, then SSL surpasses SL or UL+, respectively, only by a constant factor.

We now discuss the rate improvement of SSL for different regimes, determined by varying amounts of information about the conditional $Y|X$ captured in the unlabeled data.
For this purpose, we employ the minimax rates of~\Cref{thm:excess_main} that show an explicit dependence on $\snr, \nlab$ and $\nunl$, namely the parameters that induce these different regimes.
We use the shorthands $\ratel(\nlab, \nunl, \snr)$ and
$\rateu(\nlab, \nunl, \snr)$ to denote the rate improvement for $\thetaset=\thetaset^{(\snr)}$. 
Given~\Cref{thm:excess_main}, we can directly derive the rate improvement of SSL in the following corollary.

\begin{corollary}
  Assuming the setting of~\Cref{thm:excess_main},
    the rate improvement of SSL can be written as:
    \begin{align*}
      \textbf{Rate improvement over SL: } &\ratel\br{\nlab, \nunl, \snr} = \Theta\left(
      \frac{\nlab}{\nlab + \snr^2 \nunl}\right).\\
      \textbf{Rate improvement over UL+: } &\rateu\br{\nlab, \nunl, \snr} =\Theta\left( \frac{\snr^2
    \nunl}{\nlab + \snr^2 \nunl}\right).
    \end{align*}
    \label{cor:ssl}
\end{corollary}
 \vspace{-0.4cm}

Based on this corollary, we now argue that no SSL algorithm can simultaneously improve the rates of both SL and UL+ for any regime of $\snr, \nlab$ and $\nunl$, as summarized in Table~\ref{table:rates}. The different regimes describe asymptotically the relative amount of information about the labeling $Y|X$ that is captured in the unlabeled data.


\paragraph{1. Unlabeled data has little information relative to SNR compared to labeled data} 
For extremely low values of the SNR (i.e.\ $\snr \to 0$ faster than $\sqrt{\nunl} \to \infty$), or when labeled data is plentiful (i.e.\ $\nunl = o(\nlab)$), the relative information captured by the unlabeled data is significantly lower compared to the labeled data.
This setting has been the focus of previous worst-case analyses for
SSL \citep{BenDavid2008DoesUD, Tolstikhin2016MinimaxLB,
Gpfert2019WhenCU}. 
\Cref{cor:ssl} indicates that in this regime it holds that $\lim_{\nlab,\nunl\to\infty}\ratel(\nlab, \nunl, \snr) = c_{SL}>0$, and hence,
SSL fails to improve the SL rates even with infinite unlabeled data.

\paragraph{2. Unlabeled data is overabundant compared to labeled data} 
This setting where $\nunl \gg \nlab$ has been previously studied in works that show upper bounds for SSL algorithms.
As mentioned in~\cite{Gpfert2019WhenCU}, in order for SSL to lead to a
rate improvement compared to SL, it is \emph{necessary} that the
unlabeled set size is at least a superlinear function of the labeled
set size, i.e.~\(\nunl=\omega(\nlab)\). Corollary~\ref{cor:ssl} shows
that this condition is, in fact, \emph{sufficient} for 2-GMM
distributions: for $\snr > 0$, as long as \(\nunl\) grows
superlinearly with respect to \(\nlab\), $\ratel(\nlab, \nunl,
\snr)\to 0$, and hence, SSL can achieve faster rates than SL.  Despite
the improvement over SL, for this setting, the asymptotic error ratio
between SSL and UL+ does not vanish, i.e.\ $\lim_{\nlab, \nunl\to\infty}
\rateu\br{\nlab, \nunl, \snr} = c_{\mathrm{UL}}>0$.

\paragraph{3. Unlabeled data carries a similar amount of information as labeled data} Finally,
in the regime where the unlabeled dataset size depends linearly on the
size of the labeled set, i.e.\
\(\lim_{\nlab,\nunl\to\infty}\frac{\nunl}{\nlab}\to c\) for some
constant $c>0$, neither of the rate improvement vanishes for
$\nlab,\nunl\to\infty$.


\begin{table}\centering
  \begin{tabular}{cccc}
    \toprule
    SNR Regime& Rate of growth of \(\nunl\) vs \(\nlab\)&
    $\ratel(\nlab, \nunl, \snr)$ & $\rateu(\nlab, \nunl, \snr)$\\
    \midrule
    \(s = o\br{\sqrt{1/\nunl}}\)& Any &
    \(c_{\textrm{SL}}\)  & \(0\)\\\midrule
    
    \multirow{3}{*}{fixed \(s>0\)}& \ \(\nunl = o\br{\nlab}\) & \(c_{\textrm{SL}}\)  & \(0\)\\
    &\ \(\nunl = \omega(\nlab)\) &0&\(c_{\mathrm{UL}}\)
    \\
    &\(\lim_{\nlab,\nunl\to\infty}\frac{\nunl}{\nlab}=   c\) & \(\br{\frac{1}{1+{c}\snr^2}}c_{\mathrm{SL}}\)&\(\br{\frac{\snr^2c}{1+\snr^2c}}c_{\mathrm{UL}}\) \\\bottomrule
  \end{tabular}
  \caption{SSL rate improvement over SL and UL+ for different
  regimes of $\snr, \nlab$ and $\nunl$, where $\ratel$ and \(\rateu\) are evaluated
  for \(\lim_{\nlab, \nunl \to\infty}\).}
  \label{table:rates}
  \vspace{-0.5cm}
\end{table}

As explained in Section~\ref{sec:setting-alg}, prior UL+ algorithms that work well in practice do not use the labeled data for the unsupervised learning step. Interestingly, for the case when both the labeled and unlabeled data are used for the unsupervised step, the trends in Table~\ref{table:rates} remain the same, only 
in the regime when $\frac{\nunl}{\nlab} \to c \in (0, \infty)$ the rate improvement changes by a small constant. 
More importantly, the main takeaway from~\Cref{table:rates} remains the same: SSL cannot achieve better rates than both UL+ and SL at the same time since there is no regime for which $h_l$ and $h_u$ are simultaneously $0$.

\section{Empirical improvements over the rate-optimal SSL-S algorithm}
\label{sec:better_ssl}

\Cref{sec:theory_main} implies that the SSL minimax rates can be achieved 
by switching 
between a minimax optimal SL and a minimax optimal \ulplus algorithm (SSL-S). 
Despite exhibiting optimal statistical rates, this algorithm does not make the most effective use of the available data: SL algorithms solely use the labeled data, whereas \ulplus learns the decision boundary using just the unlabeled data, and the labeled data is only used to assign a label to the different prediction regions. Alternatively, one could use both labeled and unlabeled data to learn the decision boundary. 
In this section, we investigate the following question empirically: 
Is it possible to improve upon the error of UL+ and SL simultaneously?

First, when the unlabeled dataset is large relative to the SNR, we expect that UL+ algorithms can perform well, and hence, SSL cannot improve significantly over UL+. This is the setting most commonly considered in prior works \citep{sohn, caron2020unsupervised, grill2020bootstrapbyol}. Similarly, for large labeled datasets, it is difficult for SSL to improve the already good performance of SL algorithms. Therefore, SSL is most likely to improve over both SL and UL+ in the intermediate regime, where the information captured in the unlabeled data about the labeling $Y|X$ (referred to as ``compatibility'' in \cite{balcanblumcomp}) is moderate (i.e.\ neither too large nor too small). This intermediate regime has not been the focus of prior experimental analyses of SSL.

In this section, we present experiments to show that indeed in this intermediate regime of $\snr, \nlab$ and $\nunl$, a remarkably simple algorithm, outlined in the next paragraph,
can outperform both SL and UL+ simultaneously. 
Other standard algorithms such as self-training, which have been shown to excel in practice \citep{Xie_2020_CVPR, sohn, grill2020bootstrapbyol}, can also improve over SSL-S. 
Consequently, it remains an exciting avenue for future work to derive tight analyses that characterize the improvement of such algorithms over SL/UL+.

\subsection{Algorithms in SSL setting}

In this section, we empirically compare the prediction accuracy of SSL algorithms with minimax optimal SL and \ulplus algorithms on synthetic and real-world datasets.
We use
$\thetasl=\frac{1}{\nlab} \sum_{i=1}^{\nlab} Y_i X_i$ as the SL
estimator and an Expectation-Maximization (EM) algorithm for UL (see~\Cref{sec:appendix_exp_details} for
details).

\begin{wrapfigure}{r}{0.4\textwidth}
  \vspace{-0.1cm}
  \begin{minipage}[t]{0.4\textwidth}
        \setlength{\algomargin}{0.4cm}
      \begin{algorithm}[H]
        \DontPrintSemicolon
        \small

     \SetKwInOut{KwInput}{Input}

     \KwInput{$\Dlab$, $\Dunl$, $t$}

     \KwResult{\(\thetasslw\)}

     $\thetasl \leftarrow \algsl(\Dlab)$ \;
     
     $\thetaulp \leftarrow \algulp(\Dlab, \Dunl)$ \;

     \(\thetasslw(t) = t \thetasl + (1 - t) \thetaulp \)

     \KwRet \(\thetasslw(t)\) \;

     \caption{SSL-W algorithm}
     \label{algo:ssl_weight}

    \end{algorithm}
  \vspace{-0.3cm}
  \end{minipage}

\end{wrapfigure}

A simple approach to improve upon the naive SSL Switching algorithm
is to construct a weighted ensemble of an SL and a \ulplus estimator, trained on $\Dlab$ and $\Dunl$, respectively, with a controllable weighting hyperparameter $t$.
We call this the \textbf{SSL Weighted algorithm~(SSL-W)} shown
in~\Cref{algo:ssl_weight}. If we choose the weight $t$ according to the performance of the SL and UL+ estimators, this weighted combination can outperform both SL and UL+ (see Appendix~\ref{appendix:weighted} for a related formal statement). 
In experiments, the weight $t$ is chosen to maximize the average margin computed on an unlabeled validation set. 



Further, we also evaluate the performance of
self-training, using a procedure similar to the one analyzed in \cite{pmlr-v151-frei22a}.
 We use a logistic regression estimator for
the pseudolabeling, and train logistic regression with a ridge penalty
in the second stage of the self-training procedure. Note that an
$\ell_2$ penalty corresponds to input consistency regularization
\citep{wei2021theoreticalssl} with respect to $\ell_2$ perturbations. The ridge coefficient is also chosen using an unlabeled validation set and the procedure described above.


\subsection{Empirical improvements of SSL over SL and UL+}
\label{sec:emp-improve}


We now discuss the experimental results on two types of datasets -- synthetic and real-world data.

  

\paragraph{Synthetic data}
For a symmetric and isotropic 2-GMM
distribution $\jointsnr$ over $\RR^2$, \Cref{fig:synth-error-gap-snr} shows the gap between SSL-W and SL or~\ulplus as a function of the
SNR $\snr$. The labeled and unlabeled set sizes are  $\nlab=20$ and $n_u =2000$, respectively. An alternative way to vary the relative amount of information about the conditional $Y|X$ captured in the labeled and unlabeled datasets is via the $\nicefrac{\nunl}{\nlab}$. In~\Cref{fig:synth-error-gap-nlab-main}, we vary $\nlab$ for fixed SNR $\snr=0.5$ and $\nunl=7000$. The error of the minimax optimal switching algorithm SSL-S corresponds to the point-wise minimum of the two curves.

\begin{figure}[t]
    \centering
    \begin{subfigure}[t]{0.45\textwidth}
    \includegraphics[width=\textwidth]{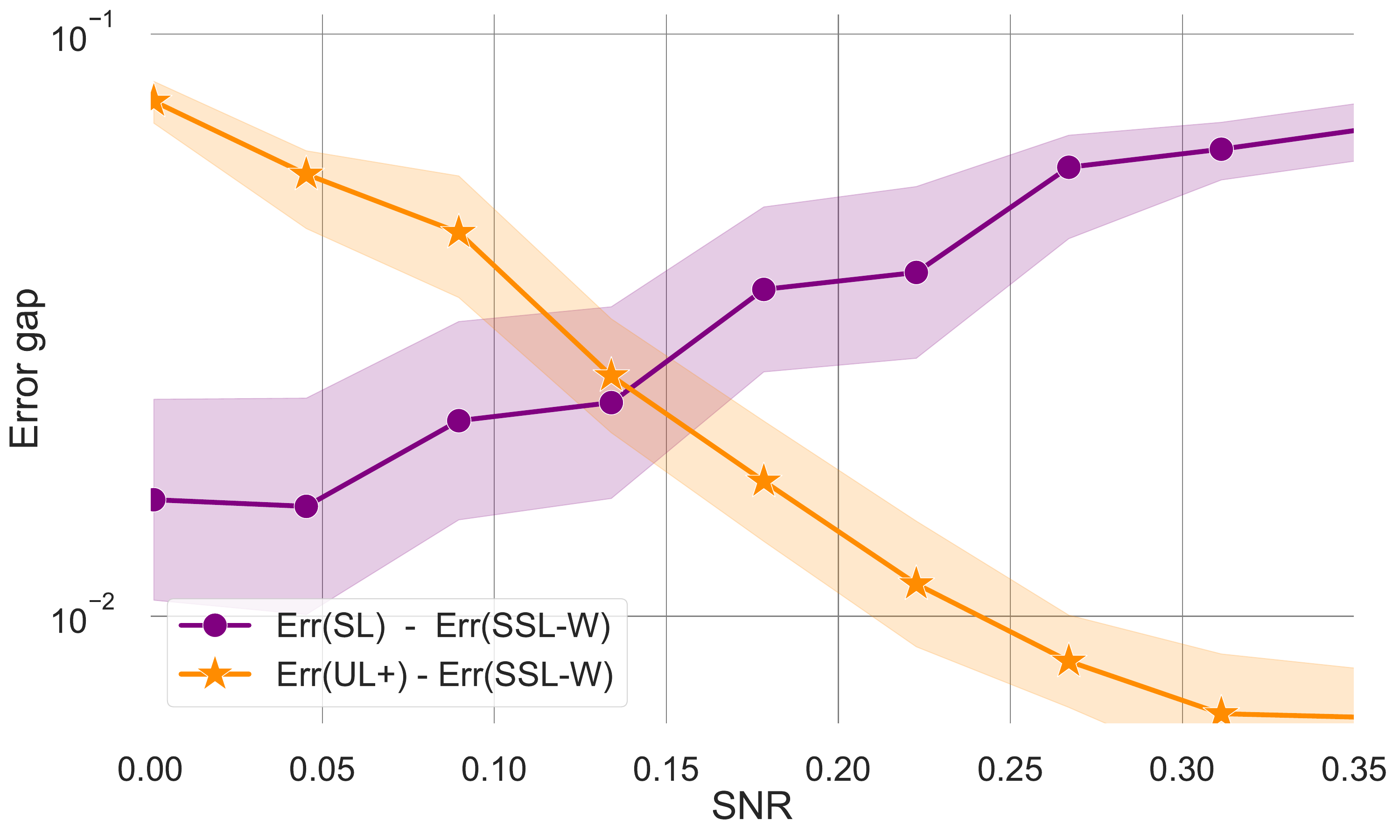}
    \caption{Varying the SNR $\snr$.}
    \label{fig:synth-error-gap-snr}
    \end{subfigure}
\begin{subfigure}[t]{0.45\textwidth}
    \centering
    \includegraphics[width=\textwidth]{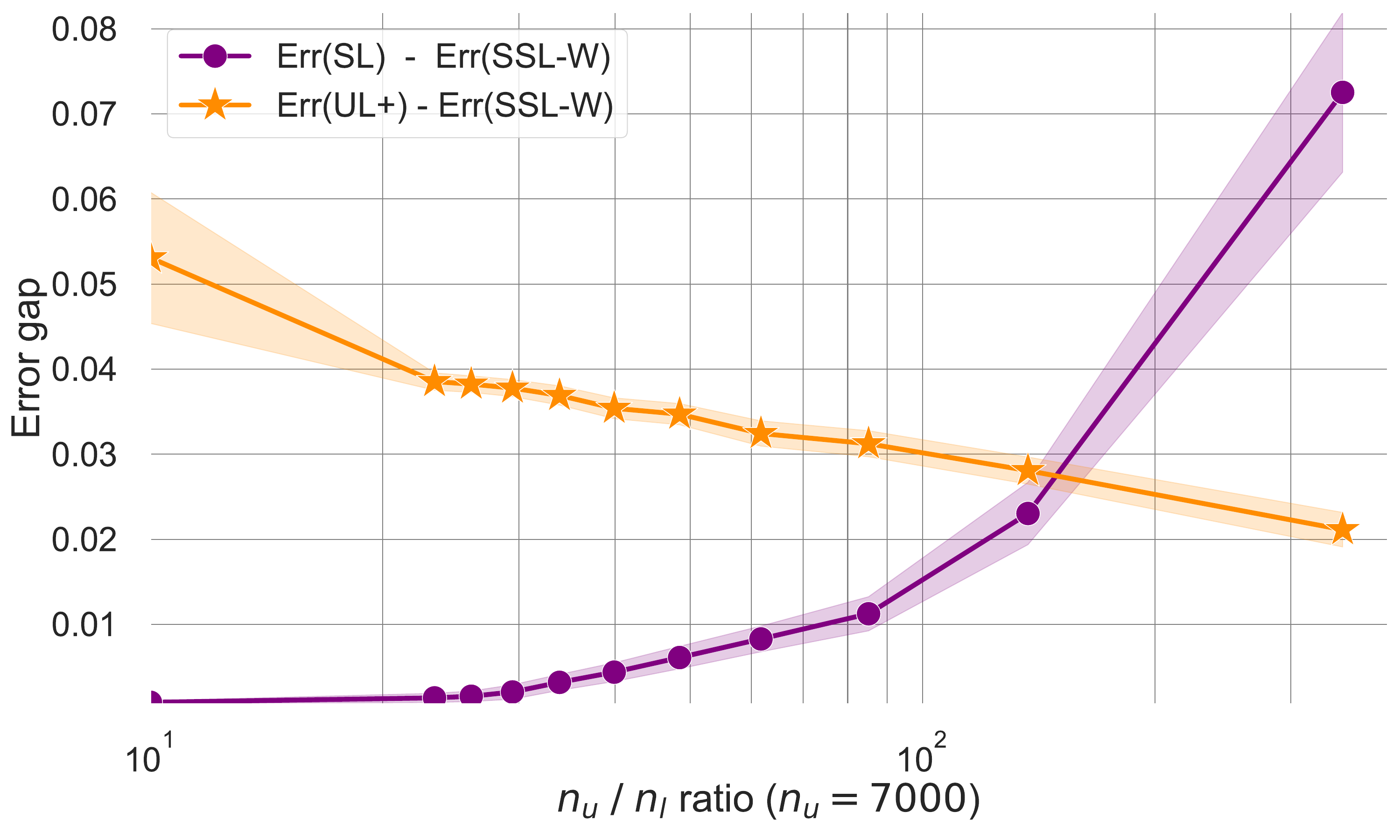}
    \caption{Varying the $\nicefrac{\nunl}{\nlab}$ ratio.}
    \label{fig:synth-error-gap-nlab-main}
\end{subfigure}

    \caption{SSL-W achieves lower error than both SL and UL+ 
    when varying the amount of information about $Y|X$ captured in the labeled and the unlabeled data. As suggested in~\Cref{table:rates}, the error of the SSL algorithm is more similar to SL for low SNR and low $\nicefrac{\nunl}{\nlab}$, and more similar to UL+ otherwise.}
    \label{fig:ssl_regimes_gmm}
\vspace{-0.4cm}
\end{figure}


First, observe that for the whole range of SNR values \(\snr\), SSL-W
always outperforms SL and UL+, and hence, also SSL-S. Further, as argued
in~\Cref{sec:analysis}, SSL-W improves more over UL+ for small values
of the SNR \(s\), and it improves more over SL for large values of the
SNR. Notably, the largest gains of SSL-W over SSL-S are in the intermediate regime of moderate SNR. 
The same trend occurs when varying the ratio $\nicefrac{\nunl}{\nlab}$ instead of the SNR, as illustrated in~\Cref{fig:synth-error-gap-nlab-main}.

\begin{wrapfigure}{r}{0.5\textwidth}
\vspace{-0.5cm}
    \centering
    \includegraphics[width=0.5\textwidth]{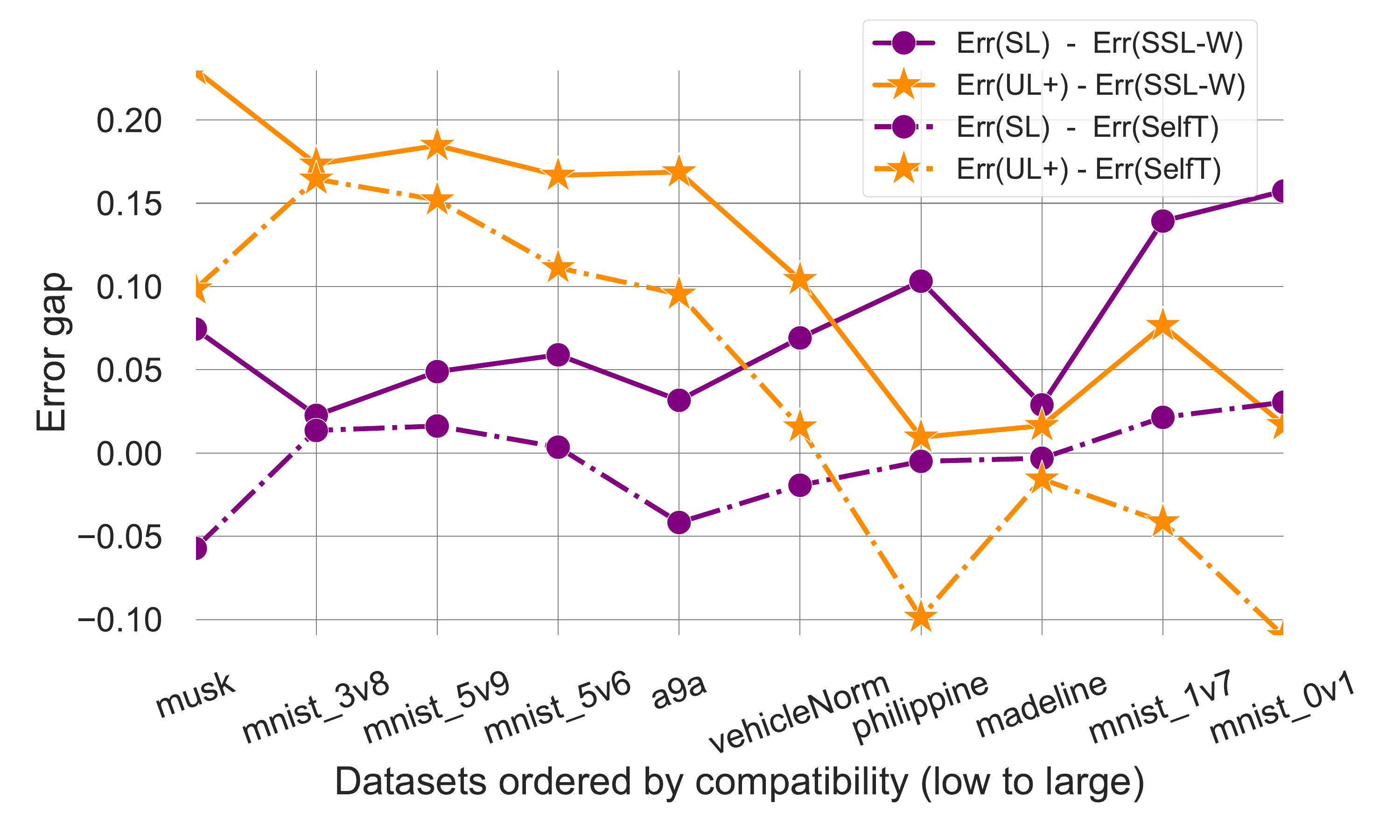}
    \caption{SSL-W and self-training, can improve upon the error of both SL and UL+ on real-world data.}
    \label{fig:real-error-gap-snr}
\vspace{-0.5cm}
\end{wrapfigure}

\paragraph{Real-world data} Next, we investigate whether similar conclusions can be drawn about SSL-W on real-world data.
We consider $10$ binary classification
real-world datasets: five from the OpenML
repository~\citep{OpenML2013} and five binary subsets of the MNIST
dataset~\citep{lecun2010}. For the MNIST subsets, we choose class
pairs that have a linear Bayes error varying between $0.1\%$ and
$2.5\%$.\footnote{We estimate the Bayes error of a dataset by training
a linear classifier on the entire labeled dataset.} From OpenML, we choose datasets that have a large enough number of samples compared to the ambient dimension
(see Appendix \ref{sec:appendix_exp_details} for
details on how we choose the datasets) and their Bayes errors vary between $3\%$ and $34\%$.

In the absence of the exact data-generating process, we quantify the
difficulty of SSL on real-world datasets using a notion of \emph{compatibility}, reminiscent of \cite{balcanblumcomp}. Specifically, we consider the compatibility given by $\rho^{-1}$ with $\rho := \frac{1}{2\sqrt{d}}\left(\misspec(\thetastar_{UL+}, \thetastar_{\text{Bayes}}) + \prederr(\thetastar_{\text{Bayes}})\right)$,
where $\misspec(\thetastar_{UL+}, \thetastar_{\text{Bayes}}) := \frac{\prederr(\thetastar_{UL+})-\prederr(\thetastar_{\text{Bayes}})}{\prederr(\thetastar_{\text{Bayes}})}$,
$d$ is the
dimension of the data, $\thetastar_{\text{Bayes}}$ is obtained via SL on the entire
dataset and $\thetastar_{\text{UL+}}$ denotes the predictor with
optimal sign obtained via UL on the entire dataset. Intuitively, this notion of compatibility captures both the Bayes error of a dataset, as well as how compatible the 2-GMM
parametric assumption actually is for the given data. 

\Cref{fig:real-error-gap-snr} shows the improvement in
classification error of both SSL-W and self-training
compared to SL and~\ulplus. The positive values of the error gap indicate that SSL-W outperforms both SL and UL+ even on real-world datasets.
This
finding suggests that the intuition
presented in this section carries over to more
generic distributions beyond just 2-GMMs, a prospect worth investigating thoroughly in future work. This conclusion is further supported by~\Cref{fig:more_switch_vs_weighted} in Appendix~\ref{sec:appendix_more_exp}, which shows that SSL-W improves over the error of the minimax optimal SSL-S algorithm, for varying values of $\nlab$. 




\section{Related work}
\label{sec:related_work}


\paragraph{General notion of when SSL can help}
\cite{balcanblumcomp} introduces a general ``compatibility'' score, 
to relate the space of
marginal distributions to the space of labeling functions.
While their
results suggest that under strong compatibility SSL may surpass the SL minimax rates, they offer no
comparison with UL/UL+.
Moreover, the paper does not discuss the minimax
optimality of the proposed SSL algorithms.

Furthermore, to characterize settings when SSL \emph{cannot} help, 
\cite{scholkopf2012} uses the causality framework to describe a family of distributions  
for which SSL algorithms do not offer any advantage over SL. In
essence, when the covariates, represented by $X$, act as causal
ancestors to the labels $Y$, the independent causal mechanism
assumption dictates that the marginal $\marginal$ offers no insights
about the labeling function.

\paragraph{SSL for regression} Beyond the theoretical works highlighted in~\Cref{sec:setting} that study SSL for classification,
\cite{10.1214/13-AOS1092azizyanssl, 10.5555/2981780.2981969SINGH}
present upper bounds for semi-supervised regression, which are
contingent on the degree to which the marginal $\marginal$ informs the
labeling function. This is akin to the results we derive in this
work. However, obtaining a minimax lower bound for semi-supervised
regression remains an exciting direction for future work. We refer to \citep{mey2020} for an overview of prior theoretical results for SSL.
On another note, even though SSL does not enhance the rates of UL,
\cite{sula2022semi} demonstrate that labeled samples can bolster the
convergence speed of Expectation-Maximization within the context of
our study.




\section{Discussion and future work}

In this paper, we establish a tight lower bound for semi-supervised
classification within the class of 2-GMM distributions. Our findings
demonstrate that across all signal-to-noise ratios, SSL \emph{cannot} simultaneously improve the error rates of both SL and UL+ -- algorithms that do not fully use the labeled and unlabeled data to learn the decision boundary.  

At the same time, we observe that empirically, an improvement is still possible for linear classification both in the 2-GMM setting and for real-world datasets. These findings suggest that future benchmarks for SSL algorithms should also consider moderate regimes, where unlabeled data is not overabundant.
Moreover, our results imply that future theoretical analyses of SSL algorithms might require careful tracking of constant factors in the analysis to show improvements over both SL and UL+.

Finally, our theoretical analysis focuses exclusively on isotropic and
symmetric GMMs as a simple proof of concept, in particular as minimax optimality had been well-established for this setting in the context of SL and UL \citep{pmlr-v54-li17a, Wu2019RandomlyIE}.
It is important to explore in future work, whether there are distributions, also beyond linear models, for which SSL can provably improve upon SL also in small $\nunl$ regime. 

\section*{Acknowledgements}

AT was supported by a PhD fellowship from the Swiss Data
Science Center. AS was supported by the ETH AI Center during his stay at ETH Zurich, where the work originated. We also thank the anonymous reviewers for their
helpful comments.
\newpage

\bibliography{arxiv_ssl.bib}

\begin{thebibliography}{10}

\bibitem{10.1214/13-AOS1092azizyanssl}
M.~Azizyan, A.~Singh, and L.~Wasserman.
\newblock {Density-sensitive semisupervised inference}.
\newblock {\em The Annals of Statistics}, 2013.

\bibitem{NIPS2013_456ac9b0}
M.~Azizyan, A.~Singh, and L.~Wasserman.
\newblock Minimax theory for high-dimensional {G}aussian mixtures with sparse
  mean separation.
\newblock In {\em Advances in Neural Information Processing Systems}, 2013.

\bibitem{balcanblumcomp}
M.-F. Balcan and A.~Blum.
\newblock A discriminative model for semi-supervised learning.
\newblock {\em Journal of the ACM}, 2010.

\bibitem{balestriero2023cookbook}
R.~Balestriero, M.~Ibrahim, V.~Sobal, A.~Morcos, S.~Shekhar, T.~Goldstein,
  F.~Bordes, A.~Bardes, G.~Mialon, Y.~Tian, et~al.
\newblock A cookbook of self-supervised learning.
\newblock {\em arXiv:2304.12210}, 2023.

\bibitem{BenDavid2008DoesUD}
S.~Ben-David, T.~Lu, and D.~P{\'a}l.
\newblock Does unlabeled data provably help? {W}orst-case analysis of the
  sample complexity of semi-supervised learning.
\newblock In {\em Annual Conference on Learning Theory (COLT)}, 2008.

\bibitem{blum98}
A.~Blum and T.~Mitchell.
\newblock Combining labeled and unlabeled data with co-training.
\newblock In {\em Annual Conference on Learning Theory (COLT)}, 1998.

\bibitem{caron2020unsupervised}
M.~Caron, I.~Misra, J.~Mairal, P.~Goyal, P.~Bojanowski, and A.~Joulin.
\newblock Unsupervised learning of visual features by contrasting cluster
  assignments.
\newblock {\em Advances in Neural Information Processing Systems}, 2020.

\bibitem{chen2020simple}
T.~Chen, S.~Kornblith, M.~Norouzi, and G.~Hinton.
\newblock A simple framework for contrastive learning of visual
  representations.
\newblock In {\em Proceedings of the International Conference on Machine
  Learning}, 2020.

\bibitem{10.5555/3495724.3497589SIMCLR}
T.~Chen, S.~Kornblith, K.~Swersky, M.~Norouzi, and G.~Hinton.
\newblock Big self-supervised models are strong semi-supervised learners.
\newblock In {\em Advances in Neural Information Processing Systems}, 2020.

\bibitem{coverandthomas}
T.~M. Cover and J.~A. Thomas.
\newblock {\em Elements of Information Theory}.
\newblock 2006.

\bibitem{el2021large}
A.~El-Nouby, G.~Izacard, H.~Touvron, I.~Laptev, H.~Jegou, and E.~Grave.
\newblock Are large-scale datasets necessary for self-supervised pre-training?
\newblock {\em arXiv:2112.10740}, 2021.

\bibitem{pmlr-v151-frei22a}
S.~Frei, D.~Zou, Z.~Chen, and Q.~Gu.
\newblock Self-training converts weak learners to strong learners in mixture
  models.
\newblock In {\em International Conference on Artificial Intelligence and
  Statistics}, 2022.

\bibitem{giraudintroduction}
C.~Giraud.
\newblock {\em Introduction to high-dimensional statistics}.
\newblock CRC Press, 2021.

\bibitem{Gpfert2019WhenCU}
C.~G{\"o}pfert, S.~Ben-David, O.~Bousquet, S.~Gelly, I.~O. Tolstikhin, and
  R.~Urner.
\newblock When can unlabeled data improve the learning rate?
\newblock In {\em Annual Conference on Learning Theory (COLT)}, 2019.

\bibitem{goyal2019scaling}
P.~Goyal, D.~Mahajan, A.~Gupta, and I.~Misra.
\newblock Scaling and benchmarking self-supervised visual representation
  learning.
\newblock In {\em International Conference on Computer Vision}, 2019.

\bibitem{grill2020bootstrapbyol}
J.-B. Grill, F.~Strub, F.~Altch{\'e}, C.~Tallec, P.~Richemond, E.~Buchatskaya,
  C.~Doersch, B.~Avila~Pires, Z.~Guo, M.~Gheshlaghi~Azar, et~al.
\newblock Bootstrap your own latent: {A} new approach to self-supervised
  learning.
\newblock {\em Advances in Neural Information Processing Systems}, 2020.

\bibitem{lecun2010}
Y.~LeCun and C.~Cortes.
\newblock {MNIST} handwritten digit database, 2010.

\bibitem{pmlr-v54-li17a}
T.~Li, X.~Yi, C.~Carmanis, and P.~Ravikumar.
\newblock Minimax {G}aussian classification \& clustering.
\newblock In {\em International Conference on Artificial Intelligence and
  Statistics}, 2017.

\bibitem{lucas2022barely}
T.~Lucas, P.~Weinzaepfel, and G.~Rogez.
\newblock Barely-supervised learning: {S}emi-supervised learning with very few
  labeled images.
\newblock In {\em AAAI Conference on Artificial Intelligence}, 2022.

\bibitem{massart2007}
P.~Massart.
\newblock {\em Concentration inequalities and model selection: Ecole d'Et{\'e}
  de Probabilit{\'e}s de Saint-Flour XXXIII-2003}.
\newblock Springer, 2007.

\bibitem{mey2020}
A.~Mey and M.~Loog.
\newblock Improvability through semi-supervised learning: {A} survey of
  theoretical results.
\newblock {\em arXiv:1908.09574}, 2020.

\bibitem{scikit-learn}
F.~Pedregosa, G.~Varoquaux, A.~Gramfort, V.~Michel, B.~Thirion, O.~Grisel,
  M.~Blondel, P.~Prettenhofer, R.~Weiss, V.~Dubourg, J.~Vanderplas, A.~Passos,
  D.~Cournapeau, M.~Brucher, M.~Perrot, and E.~Duchesnay.
\newblock Scikit-learn: Machine learning in {P}ython.
\newblock {\em Journal of Machine Learning Research}, 2011.

\bibitem{Plat1999}
J.~Platt.
\newblock Probabilistic outputs for support vector machines and comparisons to
  regularized likelihood methods.
\newblock {\em Advances in Large Margin Classifiers}, 1999.

\bibitem{ratsaby}
J.~Ratsaby and S.~S. Venkatesh.
\newblock Learning from a mixture of labeled and unlabeled examples with
  parametric side information.
\newblock In {\em Annual Conference on Learning Theory (COLT)}, 1995.

\bibitem{rigollet}
P.~Rigollet.
\newblock Generalization error bounds in semi-supervised classification under
  the cluster assumption.
\newblock {\em Journal of Machine Learning Research}, 2006.

\bibitem{scholkopf2012}
B.~Sch\"{o}lkopf, D.~Janzing, J.~Peters, E.~Sgouritsa, K.~Zhang, and J.~Mooij.
\newblock On causal and anticausal learning.
\newblock In {\em Proceedings of the International Conference on Machine
  Learning}, 2012.

\bibitem{10.5555/2981780.2981969SINGH}
A.~Singh, R.~D. Nowak, and X.~Zhu.
\newblock Unlabeled data: Now it helps, now it doesn't.
\newblock In {\em Advances in Neural Information Processing Systems}, 2008.

\bibitem{sohn}
K.~Sohn, D.~Berthelot, N.~Carlini, Z.~Zhang, H.~Zhang, C.~A. Raffel, E.~D.
  Cubuk, A.~Kurakin, and C.-L. Li.
\newblock {F}ix{M}atch: {S}implifying semi-supervised learning with consistency
  and confidence.
\newblock In {\em Advances in Neural Information Processing Systems}, 2020.

\bibitem{sula2022semi}
E.~Sula and L.~Zheng.
\newblock On the semi-supervised {E}xpectation {M}aximization.
\newblock {\em arXiv:2211.00537}, 2022.

\bibitem{Tolstikhin2016MinimaxLB}
I.~O. Tolstikhin and D.~Lopez-Paz.
\newblock Minimax lower bounds for realizable transductive classification.
\newblock {\em arXiv:1602.03027}, 2016.

\bibitem{OpenML2013}
J.~Vanschoren, J.~N. van Rijn, B.~Bischl, and L.~Torgo.
\newblock Open{ML}: {N}etworked science in machine learning.
\newblock {\em SIGKDD Explorations}, 2013.

\bibitem{wei2021theoreticalssl}
C.~Wei, K.~Shen, Y.~Chen, and T.~Ma.
\newblock Theoretical analysis of self-training with deep networks on unlabeled
  data.
\newblock In {\em International Conference on Learning Representations}, 2021.

\bibitem{Wu2019RandomlyIE}
Y.~Wu and H.~H. Zhou.
\newblock Randomly initialized {EM} algorithm for two-component {G}aussian
  mixture achieves near optimality in ${O}(\sqrt{n})$ iterations.
\newblock {\em Mathematical Statistics and Learning}, 2021.

\bibitem{Xie_2020_CVPR}
Q.~Xie, M.-T. Luong, E.~Hovy, and Q.~V. Le.
\newblock Self-training with noisy student improves {I}mage{N}et
  classification.
\newblock In {\em Proceedings of the Conference on Computer Vision and Pattern
  Recognition (CVPR)}, 2020.

\bibitem{10.3115/981658.981684YAROWSKY}
D.~Yarowsky.
\newblock Unsupervised word sense disambiguation rivaling supervised methods.
\newblock In {\em Annual Meeting on Association for Computational Linguistics},
  1995.

\bibitem{zoph2020rethinking}
B.~Zoph, G.~Ghiasi, T.-Y. Lin, Y.~Cui, H.~Liu, E.~D. Cubuk, and Q.~Le.
\newblock Rethinking pre-training and self-training.
\newblock In {\em Advances in Neural Information Processing Systems}, 2020.

\end{thebibliography}
\bibliographystyle{abbrv}

\newpage

\appendix
\section{Proof of Proposition ~\ref{prop:ul+main}}\label{appendix:ulplus}

In this section, we provide the formal statement of Proposition \ref{prop:ul+main} and its proof. 
\begin{proposition}[Fixing the sign of $\thetaul$]\label{prop:ul+}
    There exist universal constants $c_0, C_1, C_2, C_3, C_4 >0$ such that for $\nunl \geq (160/s)^2 d$, $d \ge 2$ and $\snr>0$ 
    \begin{align}
    &\EE \left[  \nrm{\thetaulp - \thetastar}\right] \leq C_1 \sqrt{\frac{d}{s^2 \nunl}} + C_2 s  e^{-\frac{1}{2} \nlab \snr^2 (1- \frac{c_0 }{\min\{\snr, \snr^2\}} \sqrt{\frac{d \log(\nunl)}{\snr^2 \nunl}})^2}, \hspace{10pt} \text{ and for $\snr \in (0, 1]$} \notag \\
    &\EE \left[ \excess (\thetaulp) \right] \leq C_3 e^{-\frac{1}{2}s^2}
         \frac{d\log (d\nunl)}{\snr^3 \nunl} + C_4 e^{-\frac{1}{2}  \snr^2 \nlab \left(1- \frac{c_0 }{ \snr^2} \sqrt{\frac{d \log(\nunl)}{\snr^2 \nunl}}\right)^2 }.   \notag
    \end{align}
\end{proposition}

\begin{proof}
    
Recall that we consider the UL+ estimator $\thetaulp = \sign\left(\thetasl^\top  \thetaul\right)\thetaul$ where 
\begin{align*}
  \thetasl &=  \frac{1}{\nlab} \sum_{i=1}^{\nlab} Y_i X_i, \text{ and } \thetaul = \sqrt{(\hat{\lambda}-1)_+}\hat{v},
\end{align*}
and define $\sgnhat:=\sign\left(\thetasl^\top  \thetaul\right)$. Now let $\sgn := \sign ({{\thetastar}^\top \thetaul }) = \argmin_{\Tilde{\sgn}\in \{-1, +1\}} \nrm{\Tilde{\sgn} \thetaul - \thetastar}^2  $ to be the oracle sign.

Note that we can write the expected estimation error of $ \thetaulp$ as
\begin{align}
    \EE \left[  \nrm{\thetaulp - \thetastar} \right] &=  \EE \left[  \nrm{\sgnhat \thetaul - \thetastar} \right] \notag 
    =  \EE \left[  \indicator{\sgnhat = \sgn} \nrm{\sgn \thetaul - \thetastar} + \indicator{\sgnhat \neq \sgn} \nrm{\sgn \thetaul + \thetastar} \right] \notag \\
    &\leq  \EE \left[  \indicator{\sgnhat = \sgn} \nrm{\sgn \thetaul - \thetastar} \right] + \EE \left[ \indicator{\sgnhat \neq \sgn} ( \nrm{\sgn \thetaul -\thetastar }+ 2\nrm{\thetastar} ) \right] \notag \\
    &\leq  \EE \left[  \left(\indicator{\sgnhat = \sgn}+ \indicator{\sgnhat \neq \sgn} \right ) \nrm{\sgn \thetaul - \thetastar} \right] + 2 \EE \left[ \indicator{\sgnhat \neq \sgn}  \right] \nrm{\thetastar} \notag \\
    &\leq \EE \left[  \nrm{\sgn \thetaul - \thetastar} \right] + 2 \snr \prob{(\sgnhat \neq \sgn)}. \label{eq:ulpluserrdecomposition}
\end{align}
We can use from \cite{Wu2019RandomlyIE} 
that $\EE \left[  \nrm{\sgn \thetaul - \thetastar}\right] \lesssim \sqrt{\frac{d}{\snr^2 \nunl}} $ and hence it remains to bound the probability of incorrectly estimating the sign (permutation).
We define a few auxiliary normally distributed variables
$\Tilde{Z} \sim \mathcal{N} (\thetaul^\top \thetastar , \frac{1}{\nlab} \|\thetaul\|^2) )$, $Z' \sim \mathcal{N} (0 , \frac{1}{\nlab}\|\thetaul\|^2 )$ and $Z\sim\mathcal{N}(0,1)$. 
Then, since $\thetasl \sim \mathcal{N} (\thetastar, \frac{1}{\nlab} I_d )$ 
, we have that
\begin{align*}
    \prob{(\sgnhat \neq \sgn)} 
    &= \prob \left( \sign (\Tilde{Z}) \neq \sign\left({\thetastar}^\top  \thetaul\right)  \right)  \\
    &\leq \prob \left( Z' \geq  |\thetaul^\top \thetastar| \right) \\
     &= \prob{\left (Z \geq \sqrt{\nlab s^2 } S_C (\thetaul, \thetastar )\right)} ,\hspace{20pt} \text{} \notag \\
\end{align*}
where $S_C (\thetaul, \thetastar ) = \frac{|\thetaul^\top \thetastar|}{\nrm{\thetaul} \nrm{\thetastar}} $. 
 Therefore, for any $A$ we have:
     \begin{align}
    \prob{(\sgnhat \neq \sgn)} & \leq \prob(Z\geq \sqrt{\nlab s^2 } (1-A)) + \prob{\left (  S_C (\thetaul, \thetastar ) \leq 1 - A \right)} \notag \\
    & \leq e^{-\frac{1}{2} \nlab \snr^2 (1-A)^2} + \prob{\left (  S_C (\thetaul, \thetastar ) \leq 1 - A \right)}, \notag
\end{align}
where we used the Chernoff bound in the last step. 
Finally,  setting $ A = \frac{c_0 }{\min\{\snr, \snr^2\}} \sqrt{\frac{d \log(\nunl)}{\snr^2 \nunl}}$ as a corollary of Proposition 6 in \cite{NIPS2013_456ac9b0} for $\nunl \geq (160/s)^2 d$ we have $\prob{\left (  S_C (\thetaul, \thetastar ) \leq 1 - A \right)} \leq \frac{d}{\nunl} $ and therefore
\begin{align}
    \prob{(\sgnhat \neq \sgn)} &\leq e^{-\frac{1}{2} \nlab \snr^2 \left(1- \frac{c_0 }{\min\{\snr, \snr^2\}} \sqrt{\frac{d \log(\nunl)}{\snr^2 \nunl}}\right)^2} + \frac{d}{\nunl}. \label{eq:prob_sign_wrong}
\end{align}
Combining this result with Equation \eqref{eq:ulpluserrdecomposition} finishes the proof of the estimation error bound in the proposition, as for some $C_1, C_2 >0 $, the following holds:
\begin{align}
     \EE \left[  \nrm{\thetaulp - \thetastar}\right] &\leq C_1 \sqrt{\frac{d}{s^2 \nunl}} + C_2 s  e^{-\frac{1}{4} \nlab \snr^2 (1- \frac{c_0 }{\min\{\snr, \snr^2\}} \sqrt{\frac{d \log(\nunl)}{\snr^2 \nunl}})^2}. \notag
\end{align}
Similarly, we can decompose the expected excess risk as follows: 
\begin{align}
    \EE \left[ \excess (\thetaulp) \right] =  \EE \left[ \excess (\sgnhat \thetaul ) \right] &=  \EE \left[  \indicator{\sgnhat = \sgn} \excess (\sgnhat \thetaul ) + \indicator{\sgnhat \neq \sgn} \excess (\sgnhat \thetaul ) \right] \notag \\
    &= \EE \left[  \indicator{\sgnhat = \sgn} \excess (\sgn \thetaul ) + \indicator{\sgnhat \neq \sgn} \excess ( - \sgn\thetaul ) \right] \notag \\
    &= \EE \left[ \excess (\sgn \thetaul ) \right] + \prob{(\sgnhat \neq \sgn)}, \label{eq:ub1}
\end{align}
where in Equation \eqref{eq:ub1} we upper bound the indicator function and the excess risk terms by 1. Finally, combining Equation \eqref{eq:prob_sign_wrong} with the upper bound in \cite{pmlr-v54-li17a} for $\EE \left[ \excess (\sgn \thetaul ) \right] $, 
we get that for some $C_3, C_4 > 0 $,
\begin{align}
    \EE \left[ \excess (\thetaulp) \right] \leq C_3 e^{-\frac{1}{2}s^2}
         \frac{d\log (d\nunl)}{\snr^3 \nunl} + C_4 e^{-\frac{1}{2}  \snr^2 \nlab \left(1- \frac{c_0 }{\min\{\snr, \snr^2\}} \sqrt{\frac{d \log(\nunl)}{\snr^2 \nunl}}\right)^2 }. \notag
\end{align}

\end{proof}

\section{Proof of SSL Minimax Rate for Estimation Error in~\Cref{thm:excess_main}}
\label{appendix:estimation}

In this section, we provide the proofs for the lower and upper bounds on the estimation error presented in Theorem \ref{thm:param_est_main}. The proof for the excess risk rates can be found in Appendix~\ref{appendix:excess}. 

\subsection{Proof of estimation error lower bound}
We first prove the estimation error lower bound in~\Cref{thm:param_est_main}. As discussed in \Cref{sec:setting}, consider the 2-GMM distributions from $\Pgmmsnr$, with isotropic components and identical covariance matrices.




   
Consider an arbitrary set of predictors $\coveringset = \{\vtheta_i\}_{i=0}^M$. 
We now apply Fano's method \citep{coverandthomas}. In particular, for any $\vtheta_0$, applying Theorem 3.1 in \citet{giraudintroduction} gives the following:
\begin{align}
\tiny
    \inf _{\algssl} &\sup _{\norm{\thetastar}=\snr}
    \EE_{\Dlab, \Dunl} \bs{\estimerr(\algssl(\Dlab, \Dunl), \jointsnr)} \notag \\
    &\geq \frac{1}{2} \min_{\substack{i,j \in [M] \\  i\neq j}} \distlb{\vtheta_i} {\vtheta_j}  \left(1-
    \frac{1+ \frac{1}{M} \sum_{i=1}^M D\left({\joint^{\, \vtheta_i}}^{\nlab} \, 
  {\marginalx^{\,\vtheta_i}}^{\nunl}\| {\joint^{\,\vtheta_0}}^{\nlab} \, 
  {\marginalx^{\, \vtheta_0}}^{\nunl}\right) }{\log (M)} \right) \notag \\
    & =  \frac{1}{2}\min_{\substack{i,j \in [M] \\  i\neq j}} \distlb{\vtheta_i} {\vtheta_j}  \left(1- \frac{1+ \frac{1}{M} \sum_{i=1}^M \nlab \kld {\joint^{\, \vtheta_i} }{\joint^{\, \vtheta_0}} + \nunl \kld {\marginalx^{\,\vtheta_i} }{ \marginalx^{\, \vtheta_0} }  }{\log (M)} \right) \label{eq:kl_independent} \\[2mm]
    &\geq  \frac{1}{2} \min_{\substack{i,j \in [M] \\  i\neq j}} \distlb{\vtheta_i} {\vtheta_j}  \left(1-
    \frac{1+ \nlab \displaystyle\max_{i \in [M]}  \kld {\joint^{\, \vtheta_i}
  }{\joint^{\, \vtheta_0}}  + \nunl \max_{i \in[M]} \kld
  {\marginalx^{\, \vtheta_i} }{ \marginalx^{\, \vtheta_0} } }{\log (M)} \right),\label{eq:kl_maximum}
\end{align}
where $\kld{\cdot}{\cdot}$ denotes the KL divergence. In Equation \eqref{eq:kl_independent}, we use the fact that the labeled and unlabeled samples are drawn i.i.d.\ from $\marginalx$ and $\joint$ and in Equation \eqref{eq:kl_maximum} we upper bound the average with the maximum. 
The next step of the proof consists in choosing an appropriate packing $\{\vtheta_i\}_{i=1}^M$ and $\vtheta_0$ on the sphere of radius $s$, i.e. $\frac{1}{s} \vtheta_i \in S^{d-1}$, that optimizes the trade-off between the minimum and the maxima in Equation \eqref{eq:kl_maximum}.

For the packing, we use the same construction that was employed by \cite{Wu2019RandomlyIE} for deriving adaptive bounds for unsupervised learning. This construction 
also turns out to lead to a tight lower bound for the semi-supervised 
setting. 
Let $c_0$ and $C_0$ be positive absolute constants and let $\Tilde{\coveringset} = \{\psi_1,..., \psi_M\}$ be a $c_0$-net on the unit sphere $S^{d-2}$ such that we have $|\Tilde{\coveringset}| = M\geq e^{C_0d}$. For an absolute constant $\alpha \in [0, 1]$ 
, we construct the following packing of the 
sphere of radius $s$ in $\R^d$:
    \begin{equation}
        \coveringset =\left \{\vtheta_i =  \snr 
        \begin{bmatrix}
        \sqrt{1-\alpha^2} \\
        \alpha \psi_i
        \end{bmatrix} \middle|  \psi_i \in \Tilde{\coveringset} 
        \right\}, \notag
    \end{equation}
and define $\vtheta_0 = [\snr, 0, ..., 0]$. Note that, by definition, $\distlb{\vtheta_i}{\vtheta_j} \geq c_0 \snr \alpha $, 
for any distinct $i,j\in[M]$, which lower bounds the first term in \eqref{eq:kl_maximum}. Furthermore, $\distlb{\vtheta_i} {\vtheta_0} \leq \sqrt{2} \alpha \snr$, for all $i \in [M]$.

We can now upper bound the maxima in  Equation \eqref{eq:kl_maximum} for $\vtheta_i$ in the packing. For the KL divergence between two GMMs with identity covariance matrices, we obtain 
\begin{equation}
\label{eq:jointkl}
        \kld{\joint^{\, \vtheta_i}}{\joint^{\, \vtheta_0}} = \frac{1}{2} \nrm{\vtheta_i - \vtheta_0}^2_2 \leq \alpha^2 s^2, \text{ for all }i=[M]. 
\end{equation}
Further, we can upper bound the KL divergence between marginal distributions, namely $\kld {\marginalx^{\vtheta_i} }{
\marginalx^{\vtheta_0} }$, using Lemma 27 in \cite{Wu2019RandomlyIE} that yields
\begin{equation}
    \label{eq:marginalkl}
        \max_{i\in [M]}\kld{\marginalx^{\,  \vtheta_i}}{\marginalx^{\,  \vtheta_0}} \leq C \max_{i\in[M]}\distlb{\frac{1}{s}\vtheta_i}{\frac{1}{s}\vtheta_0}^2 \snr^4 \leq 2 C \alpha^2 \snr^4. 
\end{equation}

Plugging Equations~\eqref{eq:jointkl} and \eqref{eq:marginalkl} into Equation \eqref{eq:kl_maximum}, we finally obtain the following lower bound
for the minimax error, which holds for any $\alpha \in [0,1]$:
\begin{align}
    \inf _{\algssl} &\sup _{\norm{\thetastar}=\snr}
    \EE_{\Dlab, \Dunl} \bs{\estimerr(\algssl(\Dlab, \Dunl), \thetastar)} 
   \geq \frac{1}{2} c_o \alpha \snr \left(1- \frac{1 + \nlab \snr^2\alpha^2 +
   \nunl C_1 \snr^4\alpha^2 }{C_0 d} \right). \notag 
\end{align}
%
%
%
Using the optimal value 
$\alpha =\min \left\{ 1, \sqrt{\frac{C_0 d - 1 }{ 3 \snr^2 \nlab+ 3 C_1
\snr^4 \nunl}} \right\}$ then concludes the proof.




\subsection{Proof of estimation error upper bound}
\label{sec:est-ub-proof}



We now prove the tightness of our lower bound by establishing the upper bound for the estimation error of the SSL Switching algorithm presented in
Algorithm~\ref{algo:ssl_switch} for a particular choice of the minimax optimal SL and UL+ estimators \cite{Wu2019RandomlyIE}. Let
\begin{align}
  \thetasl &=  \frac{1}{\nlab} \sum_{i=1}^{\nlab} Y_i X_i \notag\\
  \thetaulp &= \sign\left(\thetasl^\top  \thetaul\right)  \thetaul \text
  {, with } \thetaul = \sqrt{(\hat{\lambda}-1)_+}\hat{v},\label{eq:ulplus}
\end{align}
where $(\hat{\lambda}, \hat{v})$ is the leading eigenpair of the sample
covariance matrix $\hat{\Sigma} = \frac{1}{\nunl} \sum_{j=0}^{\nunl} X_j
X_j^T$ and we use the notation $(x)_+:=\max(0, x)$. 
The upper bound for the expected error incurred by the UL+ estimator is given in~\Cref{prop:ul+}.




For the SL estimator $\thetasl$, we apply standard results for Gaussian distributions to upper bound the estimation error that holds for any regime of $n$ and $d$:
  \begin{align}
    \label{eq:sl_ub}
    \EE_{\Dlab} \bs{\|\thetasl -
    \thetastar\|} \le \sqrt{\frac{d}{\nlab}}.
  \end{align}
Using Equation~\eqref{eq:sl_ub} and Proposition~\ref{prop:ul+} and switching
between $\thetasl$ and $\thetaulp$ according to the conditions in
Algorithm~\ref{algo:ssl_switch} that pick the better performing of the two algorithms depending on the regime,
we can show that
there exist universal constants $C, c_0>0$ such that for $0\leq \snr \leq 1$, $d\geq 2$ and $\nunl \geq (160/s)^2 d$,
it holds that
\begin{equation}
\label{eq:sslupper}
        \EE \left[  \nrm{\thetassls - \thetastar} \right] \leq  C 
        \min \left\{\snr, {\sqrt{\frac{d}{\nlab}}}, \sqrt{\frac{d}{s^2 \nunl}} + s  e^{-\frac{1}{2} \nlab \snr^2 \left(1- \frac{c_0 }{\snr^2}\sqrt{\frac{d \log(\nunl)}{\snr^2 \nunl}}\right)^2} \right\} ,
\end{equation}
where the expectation is over $\Dlab \sim
    \left(\jointsnr\right)^{\nlab}$ and $\Dunl \sim
  \left(\marginalsnr\right)^{\nunl}$.
  
\paragraph{Matching lower and upper bound.}

When $\nlab > O(\frac{\log (\nunl)}{\snr^2})$, the exponential term becomes negligible and the first additive component dominates in the last term in the right-hand side of Equation~\eqref{eq:sslupper}.
Basic calculations then yield that the expected error of the switching algorithm is upper bounded by $C' \min\left\{s,  \sqrt{\frac{d}{\nlab+\snr^2 \nunl}}\right\}$ 
for some constant $C'$, which concludes the proof of the theorem.

\section{Proof of SSL Minimax Rate for Excess Risk in~\Cref{thm:excess_main}} \label{appendix:excess}

In this section, we prove the minimax lower bound on excess risk for
an algorithm that uses both labeled and unlabeled data and a
matching~(up to logarithmic factors) upper bound.

\subsection{Proof of excess risk lower bound}
We first prove the excess error minimax lower bound
in~\Cref{thm:excess_main}, namely we aim to show that there exists a constant $C_0 > 0$ such
that for any $\snr>0$, $\nunl, \nlab\geq 0$ and $d \geq 4$,  we have
\begin{equation}\label{eq:lb-est-err-main}
    \inf_{\algssl} \sup _{\left\|\thetastar\right\|=\snr} \EE \left[ \errssl
    \right] \geq C_0 e^{-\snr^2/2} \min \left\{  \frac{d}{\snr \nlab + \snr^3
  \nunl} , s\right\},
\end{equation}
where the expectation is over $\Dlab \sim
\left(\jointsnr\right)^{\nlab}$ and $\Dunl \sim
\left(\marginalsnr\right)^{\nunl}$. Our approach to proving this lower
bound is again to apply Fano's method~\citep{giraudintroduction} using the
excess risk as the evaluation method. The reduction from estimation to testing usually
hinges on the triangle inequality in a metric space. 
Since the excess risk does not satisfy the metric axioms, we employ a technique introduced in \citet{NIPS2013_456ac9b0} to derive an alternative sufficient condition for applying Fano's inequality.

Let $\vtheta_1, \ldots, \vtheta_M \in \thetaset$, $M \geq 2$, and $\gamma>0$. If for all $1 \leq i \neq j \leq M$ and $\thetahat$, 
\begin{equation}
\label{eq:packass}
     \excess \left( \thetahat, \vtheta_i \right)
     <\gamma \text {\,\,   implies  \,\, } \excess \left( \thetahat, \vtheta_j \right)
     \geq \gamma,
\end{equation}
  then
  \begin{align}
      \inf _{\algssl} \max _{i \in[0 . . M]} \EE &\left[
      \excess \left( \algssl(\Dlab,\Dunl), \vtheta_i \right) \right]  \\ 
      &\geq \gamma  
      \left(1 - \frac{1+ \nlab \displaystyle\max_{i \neq j}  \kld {\joint^{\, \vtheta_i} }{\joint^{\, \vtheta_j}}  + \nunl \max_{i \neq j} \kld {\marginalx^{\, \vtheta_i} }{ \marginalx^{\, \vtheta_j} } }{\log(M)} \right), \notag
  \end{align}
where the expectation is over $\Dlab \sim
\left(\joint^{\, \vtheta_i} \right)^{\nlab}$ and $\Dunl \sim
\left(\marginal^{\, \vtheta_i} \right)^{\nunl}$. 

In order to get a lower bound,
we again pick $\vtheta_i,\dots,\vtheta_M$ to be an appropriate packing,  so that the condition in ~\Cref{eq:packass} can be satisfied. 
For this purpose, we can simply use the construction from \cite{pmlr-v54-li17a}, which was previously used to obtain tight bounds for supervised and unsupervised settings. 
Let $\sparsparam = (d-1)/6 $. By Lemma
4.10 in~\cite{massart2007}, there exists a set $\Tilde{\coveringset}
= \{\psi_1, \ldots, \psi_{M}\}$, such that $\nrm{\psi_i}_0 =
\sparsparam $, $\psi_i \in \{0, 1\}^{d-1}$, the Hamming distance
$\delta\left(\psi_i, \psi_j\right)> \sparsparam / 2$ for all $1 \leq
i<j \leq M = |\Tilde{\coveringset}| $, and $\log M \geq
\frac{\sparsparam}{5} \log \frac{d}{\sparsparam} \geq d \log(6) /60 =
c_1 d $.

Define
\begin{equation}
        \coveringset =\left \{\vtheta_i = 
        \begin{bmatrix}
        \sqrt{\snr^2 - p \alpha^2}\\
        \alpha \psi_i
        \end{bmatrix} \middle|  \psi_i \in \Tilde{\coveringset} 
        \right\} \notag
    \end{equation}
for some absolute constant $\alpha$. Note that since $\nrm{\vtheta_i} = \snr$ and $\nrm{\vtheta_i - \vtheta_j}^2 = \alpha^2 \delta\left(\psi_i, \psi_j\right)$, we have  
\begin{equation}\label{eq:theta_i_j_diff}
    \frac{\sparsparam \alpha^2}{2} \leq \nrm{\vtheta_i - \vtheta_j}^2 \leq 2 \sparsparam \alpha^2 
\end{equation}
and  
\begin{equation}\label{eq:theta_i_j_ip}
    \snr^2 -\sparsparam \alpha^2 \leq \vtheta_i^\top \vtheta_j \leq \snr^2 -\sparsparam \alpha^2/4 . 
\end{equation}
First, we show that the excess risk satisfies the condition in~\Cref{eq:packass}.
As in the proof of Theorem~1 in~\cite{pmlr-v54-li17a}, we
have that for any \(\vtheta\), 
\begin{align}
    \excrisk(\vtheta, \vtheta_i) +  \excrisk(\vtheta, \vtheta_j) 
    \geq 2c_0 e^{-\nicefrac{\snr^2}{2}} \frac{\sparsparam \alpha^2 }{\snr}, \notag
\end{align}
and thus, for all $i$ and $j\neq i$, it holds that
\begin{equation}\label{eq:excess-gen-hyp}
    \excrisk(\vtheta, {\vtheta_i}) \leq c_0 e^{-\nicefrac{\snr^2}{2}} \frac{\sparsparam \alpha^2 }{\snr} \implies \excrisk(\vtheta, \vtheta_j) \geq c_0 e^{-\nicefrac{\snr^2}{2}} \frac{\sparsparam \alpha^2 }{\snr}. 
\end{equation}
Then since the condition in~\Cref{eq:packass} is satisfied, we
obtain
\begin{equation}
    \label{eq:fano-general}
    \begin{aligned}
    \inf_{\algssl} &\sup _{\left\|\thetastar\right\|=\snr} \EE_{\Dlab, \Dunl} \left[ \errssl \right] 
     \geq \inf _{\algssl} \max _{i \in[0 . . M]} \EE \left[
      \excess \left( \algssl(\Dlab,\Dunl),  \vtheta_i \right) \right]  \\ 
   &\geq c_0 e^{-\nicefrac{\snr^2}{2}} \frac{\sparsparam \alpha^2 }{\snr} \left(1 - \frac{1+ \nlab \displaystyle\max_{i \neq j}  \kld {\joint^{\, \vtheta_i} }{\joint^{\, \vtheta_j}}  + \nunl \max_{i \neq j} \kld {\marginalx^{\, \vtheta_i} }{ \marginalx^{\, \vtheta_j} } }{\log(M)} \right).
    \end{aligned}
\end{equation}
Next, we bound the KL divergences between the two joint distributions
and between the two marginals that appear in~\Cref{eq:fano-general}. For the joint distributions we have that:
\begin{equation}\label{eq:kl-joint}
\begin{aligned}
    \kld {\joint^{\, \vtheta_i} }{\joint^{\, \vtheta_j}} &= \frac{1}{2}  \nrm{\vtheta_i - \vtheta_j}_2^2 \leq \sparsparam  \alpha^2,
\end{aligned}
\end{equation}
where the inequality follows from~\Cref{eq:theta_i_j_diff}. Using Proposition~24 in~\cite{NIPS2013_456ac9b0}, we bound the KL divergence between the two marginals
\begin{equation}\label{eq:kl-marginal}
\begin{aligned}
    \kld {\marginalx^{\, \vtheta_i} }{\marginalx^{\, \vtheta_j}} &\lesssim \snr^4 \left( 1- \frac{\vtheta_i^\top \vtheta_j}{\nrm{\vtheta_i} \nrm{\vtheta_j}}\right)  \leq \sparsparam \snr^2 \alpha^2,
\end{aligned}
\end{equation}
where the inequality follows from~\eqref{eq:theta_i_j_ip}.
Plugging~\eqref{eq:kl-joint} and~\eqref{eq:kl-marginal} into~\eqref{eq:fano-general} and setting $$\alpha^2 = c_3 \min \left \{ \frac{c_1 d - \log 2}{8 (\sparsparam \nlab +\snr^2 \sparsparam \nunl)}, \frac{\snr^2}{\sparsparam} \right \}, $$

gives the desired result
\begin{equation}
     \inf_{\algssl} \sup _{\left\|\thetastar\right\|=\snr} \EE_{\Dlab, \Dunl} \left[ \errssl
    \right] \gtrsim e^{-\snr^2/2} \min \left\{  \frac{d}{\snr \nlab + \snr^3 \nunl} , s\right\}. \notag
\end{equation}

\subsection{Proof of excess risk upper bound}   
Next, we prove the upper bound on the excess risk of the SSL switching
estimator \(\thetassls\) output by Algorithm~\ref{algo:ssl_switch}
with the supervised and unsupervised estimators defined
in~\Cref{sec:est-ub-proof}  to show the tightness
of the excess risk lower bound in~\Cref{thm:excess_main}.  
{Using same arguments as in \Cref{sec:est-ub-proof}, excess risk upper bound from \citet{pmlr-v54-li17a} and \Cref{prop:ul+}}
there exist universal constants $C, c_0>0$ such that for $\snr >0$, 
$d\geq 2$ and for sufficiently large $\nunl$ and $\nlab$,
    \begin{equation}
        \EE \left[ \excess(\thetassls) \right] \leq  C e^{-\frac{1}{2}s^2}
        \min \left\{\snr, {\frac{d \log (\nlab)}{\snr \nlab}}, \frac{d\log (d\nunl)}{\snr^3 \nunl} +  e^{-\frac{1}{2}  \snr^2 \left( \nlab \left(1- \frac{c_0 }{\min\{\snr, \snr^2\}} \sqrt{\frac{d \log(\nunl)}{\snr^2 \nunl}}\right)^2 - 1 \right)} \right\} ,\notag
    \end{equation}
where the expectation is over $\Dlab \sim
    \left(\jointsnr\right)^{\nlab}$ and $\Dunl \sim
  \left(\marginalsnr\right)^{\nunl}$.

\section{Theoretical motivation for the SSL Weighted Algorithm}
\label{appendix:weighted}

In this section, we formalize the intuition how SSL-W 
introduced in~\Cref{sec:better_ssl} can achieve lower squared estimation error (up to sign permutation) compared to SSL-S. 
For this purpose we consider a slightly different SSL-W estimator compared to the one introduced in Section~\ref{sec:better_ssl} 
$$\thetasslwperm(t) := t \thetasl + (1 - t) \thetaulperm$$
which uses the UL estimator 
$\thetaulperm$ with the oracle sign, namely $\thetaulperm := \argmin_{\vtheta \in \{\thetaul, -\thetaul\}} \EE\bs{\norm{\vtheta-\thetastar}^2}$. Note that this UL estimator is not using labeled data to determine the sign. For a fair comparison, we compare $\thetasslwperm$ with the (oracle) switching estimator that also has access to the UL estimator with correct sign 
$$\thetassls^*:= \argmin_{\vtheta \in \{\thetasl, \thetaulperm\}} \EE\bs{\norm{\vtheta-\thetastar}^2}.$$

\begin{figure}[t]\centering
    \includegraphics[width=0.5\linewidth]{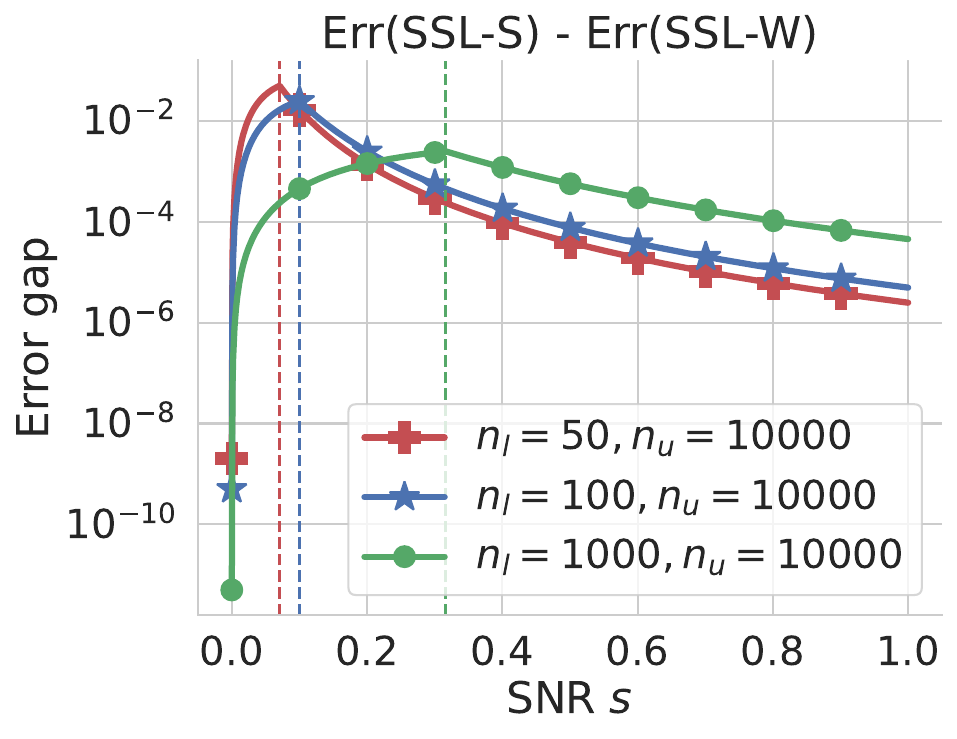}
    \vspace{-0.1cm}
    \caption{Estimation error gap between SSL-S and SSL-W as
        revealed by~\Cref{thm:param_est_diff_switch_weight} for
        varying SNR and $\nlab$ ($\nunl=10000$). The maximum gap is
        reached at the switching point, indicated by the vertical
        dashed lines.}
    \label{fig:sslw_theory}
    \vspace{-0.5cm}
\end{figure}

\Cref{thm:param_est_diff_switch_weight} now shows that there exists an optimal weight $t^*$ for which the SSL-W predictor achieves lower
estimation error than the SSL Switching predictor, 
$\thetassls^*$. Intuitively, it is simply due to the fact that SSL-S only uses either $\nlab$ or $\nunl$, while SSL-W actively uses both datasets. 

Note that \Cref{thm:param_est_diff_switch_weight}  holds for arbitrary distributions and estimators $\thetasl$ and $\thetaulperm$
as long as they are independent and one of them is unbiased. Therefore, future results that
derive upper bounds for SL and UL+ for other distributional
assumptions and estimators can seamlessly be plugged
into~\Cref{thm:param_est_diff_switch_weight}. By the same argument, \(\thetaweight\)
obtained by other SL and UL+ estimators can also be expected to
improve over the respective SL and UL+ estimators, given that one of
them is unbiased.

\begin{restatable}{proposition}{thmimprovementweight}  \label{thm:param_est_diff_switch_weight} 

    Consider a distribution \(\jointsnr\in \Pgmmsnr\) and let \(d\geq
    2\), and $\nlab, \nunl > 0$. Let $\thetasslwperm(t^*)$ be the SSL-W estimator introduced above. Then there exists a weight $t^*\in(0, 1)$
    for which
    \begin{align}
      \label{eq:gap}
      \EE\bs{\norm{\thetassls^*-\thetastar}^2} - \EE\bs{\norm{\thetasslwperm(t^*)-
        \thetastar}^2} =  \min \bc{r, \frac{1}{r}}
        \EE\bs{\norm{\thetasslwperm(t^*)- \thetastar}^2},
    \end{align}
    where $r=\frac{\EE\bs{ \norm{\thetaulperm - \thetastar}^2}}{\EE\bs{
\norm{\thetasl - \thetastar}^2}}$, and the expectations are over
\(\Dlab\sim\br{\jointsnr}^{\nlab},\Dunl\sim\br{\marginalsnr}^{\nunl}\).
    \end{restatable}
    
 Since the RHS of Equation~\eqref{eq:gap} is always positive,
  $\thetasslwperm(t^*)$ always outperforms $\thetaswitch^*$ as long as the
  conditions of~\Cref{thm:param_est_diff_switch_weight} are satisfied.
  The magnitude of the error gap between SSL-S and SSL-W depends on
  the gap between SL and UL+ (see Figure~\ref{fig:sslw_theory}). The
  maximum gap is reached for $\EE\bs{\norm{\thetaulperm - \thetastar}^2}
  \approx \EE\bs{\norm{\thetasl - \thetastar}^2}$ when SSL-W obtains
  half the error of SSL-S.

\begin{proof}
    


The first step in
proving~\Cref{thm:param_est_diff_switch_weight} is to express the
estimation error of $\thetasslwperm(t^*)$ in terms of the estimation
errors of $\thetasl$ and $\thetaulperm$.

Let $\thetahat_W{(t^*)} = {{t^*}} {\thetahat_1} + {(1-{t^*})} {\thetahat_2}$. By definition of \(\thetahat_W{(t^*)}\), for any two independent estimators $\thetahat_1,\thetahat_2$ and unbiased $\thetahat_1$ we have
    \begin{align*}
        \EE \left[\nrm{\thetahat_W{(t^*)} - \thetastar}^2 \right]
        = \EE\left[ {{t^*}}^2 \nrm{\thetahat_1 - \thetastar}^2 + {(1-{t^*})}^2 \nrm{\thetahat_2 - \thetastar}^2\right].
    \end{align*}
Plugging in  $t^* = \frac{\EE \bs{ \nrm{\thetahat_2 -
    \thetastar}^2} } { \EE \bs{\nrm{\thetahat_1 - \thetastar}^2} +
    \EE \bs{\nrm{\thetahat_2 - \thetastar}^2}}$  then yields that
 \[\EE
    \bs{\nrm{\thetahat_W{(t^*)} - \thetastar}^2} = \left(\frac{1}{\EE
    \bs{\nrm{\thetahat_1 - \thetastar}^2 }}+\frac{1}{\EE
    \bs{\nrm{\thetahat_2 - \thetastar}^2}}\right)^{-1}.\] 

Plugging in $\thetahat_1 = \thetasl$ and $\thetahat_2 = \thetaulperm$ we have $\EE
    \bs{\nrm{\thetasslwperm - \thetastar}^2} = \left(\frac{1}{\EE
    \bs{\nrm{\thetasl - \thetastar}^2 }}+\frac{1}{\EE
    \bs{\nrm{\thetaulperm - \thetastar}^2}}\right)^{-1}$. The proof then follows from calculating the difference between the
harmonic mean and the minimum of estimation errors of $\thetasl$ and
$\thetaulp$. Let $x,y \in \R_{+}$ and w.l.o.g. assume $x \leq y$. Then we have:
    \begin{equation}
        x - \left(\frac{1}{x} + \frac{1}{y} \right)^{-1} = x - \frac{xy}{x+y} = \frac{x^2}{x+y} = \frac{x}{y} \frac{xy}{x+y}. \notag
    \end{equation}
Applying this identity concludes the proof after choosing   
    $x= \min \left\{ \EE [
\nrm{\thetaulperm - \thetastar}^2] , \EE [
\nrm{\thetasl - \thetastar}^2] \right \}$ and \(y= \max \left\{ \EE [
\nrm{\thetaulperm - \thetastar}^2] , \EE [
\nrm{\thetasl - \thetastar}^2] \right \}\).

\end{proof}
\section{Simulation details}
\label{sec:appendix_exp_details}

\subsection{Methodology}

We split each dataset in a test set, a validation set and a training set. Unless otherwise specified, the
unlabeled set size is fixed to $5000$ for the synthetic experiments and $4000$
for the real-world datasets. The size of the labeled set $\nlab$ is varied in
each experiment. For each dataset, we draw a different labeled subset $20$ times
and report the average and the standard deviation of the error gap (or the
error) over these runs. The (unlabeled) validation set and the test set have $1000$ labeled
samples each.

We use logistic regression from Scikit-Learn \citep{scikit-learn} as the supervised
algorithm. We use the validation set to select the ridge penalty for SL. For the
unsupervised algorithm, we use an implementation of Expectation-Maximization
from the Scikit-Learn library. We also use the self-training algorithm from
Scikit-Learn with a logistic regression estimator. The best confidence threshold
for the pseudolabels is selected using the validation set. Moreover, the optimal
weight for SSL-W is also chosen with the help of the validation set. Since we use an unlabeled validation set, we need to employ an appropriate criterion for hyperparameter selection. Therefore, we select models that lead to a large (average) margin measured on the unlabeled validation set. 
We give
SSL-S the benefit of choosing the optimal switching point between SL and UL+ by
using the test set. Even with this important advantage, SSL-W (and sometimes
self-training) still manage to outperform SSL-S.

\subsection{Details about the real-world datasets}

\paragraph{Tabular data.} We select tabular datasets from the OpenML repository
\citep{OpenML2013} according to a number of criteria.  We focus on
high-dimensional data with $100 \le d < 1000$, where the two classes are not
suffering from extreme class imbalance, i.e.\ the imbalance ratio between the
majority and the minority class is at most $5$. Moreover, we only consider
datasets that have substantially more samples than the number of features, i.e.\
$\frac{n}{d} > 10$. In the end, we are left with 5 datasets, that span a broad
range of application domains, from ecology to chemistry and finance.

To assess the difficulty of the datasets, we train logistic regression on the
entire data that is available, and report the training error. Since we train on
substantially more samples than the number of dimensions, the predictor that we
obtain is a good estimate of the linear Bayes classifier for each dataset.

Furthermore, we measure the extent to which the data follows a GMM distribution
with spherical components. We fit a spherical Gaussian to data coming from each
class and use linear discriminant analysis (LDA) for prediction. We record the
training error (of the best permutation). Intuitively, this is a score of how
much our assumption about the connection between the marginal $\marginal$ and
the labeling function $P(Y|X)$ is satisfied. For
Figure~\ref{fig:real-error-gap-snr} we take $\thetastar_{\text{Bayes}}$ to be the linear Bayes classifier and
$\thetastar_{UL}$ the LDA classifier described above.\footnote{Note that we
refer to the LDA estimator as \emph{UL} since we use it as a proxy to assess how
well unsupervised learning can perform on each dataset.} If the data is almost linearly separable (i.e.\
$\prederr(\thetastar_{\text{Bayes}})\le 0.01$), then we simply take the linear Bayes error
as the compatibility.

\paragraph{Image data.} In addition to the tabular data, we also consider a number of datasets that
are subsets of the MNIST dataset \citep{lecun2010}. More specifically, we create
binary classification problems by selecting class pairs from MNIST. We choose
$5$ classification problems which vary in difficulty, as measured by the Bayes
error, from easier (e.g.\ digit 0 vs digit 1) to more difficult (e.g.\ digit 5
vs digit 9).  Table~\ref{table_stats} presents the exact class pairs that we
selected. To make the problem more amenable for linear classification, we
consider as covariates the $20$ principle components of the MNIST images.

\begin{table}[h]
\scriptsize
\centering

\begin{tabular}{lrrr}
\toprule
Dataset name &    $d$ & Linear classif. training error & LDA w/ spherical GMM training error\\
\midrule
mnist\_0v1  & 784 & 0.001& 0.009\\
mnist\_1v7  & 784 & 0.006& 0.036\\
madeline    & 259 & 0.344& 0.381\\
philippine  & 308 & 0.240& 0.318\\
vehicleNorm & 100 & 0.141& 0.177\\
mnist\_5v9  & 784 & 0.024& 0.045\\
mnist\_5v6  & 784 & 0.024& 0.042\\
a9a         & 123 & 0.150& 0.216\\
mnist\_3v8  & 784 & 0.042& 0.105\\
musk        & 166 & 0.037& 0.270\\
\bottomrule
\end{tabular}

\caption{Some characteristics of the datasets considered in our
experimental study.}
\label{table_stats}

\vspace{-0.5cm}
\end{table}


\section{More experiments}
\label{sec:appendix_more_exp}

In this section we present further experiments that indicate that the SSL
Weighted algorithm (SSL-W) can indeed outperform the minimax optimal 
Switching algorithm (SSL-S). 
Figure~\ref{fig:more_switch_vs_weighted} shows that there exists a positive error gap between SSL-S and SSL-W for a broad range of $\nlab$ values, for real-world data. The extent of the error
gap is determined by the $\frac{\nunl}{\nlab}$ ratio as well as the
signal-to-noise ratio that is specific to each data distribution. In addition, \Cref{fig:more_switch_vs_weighted} also
indicates that self-training can outperform SSL-W in some scenarios. It remains an exciting direction
for future work to provide a tight analysis of self-training that can indicate when
it outperforms both SL and UL+ simultaneously.

  

\begin{figure*}[t]
  \centering
  \begin{subfigure}[t]{0.32\textwidth}
    \centering
    \includegraphics[width=\textwidth]{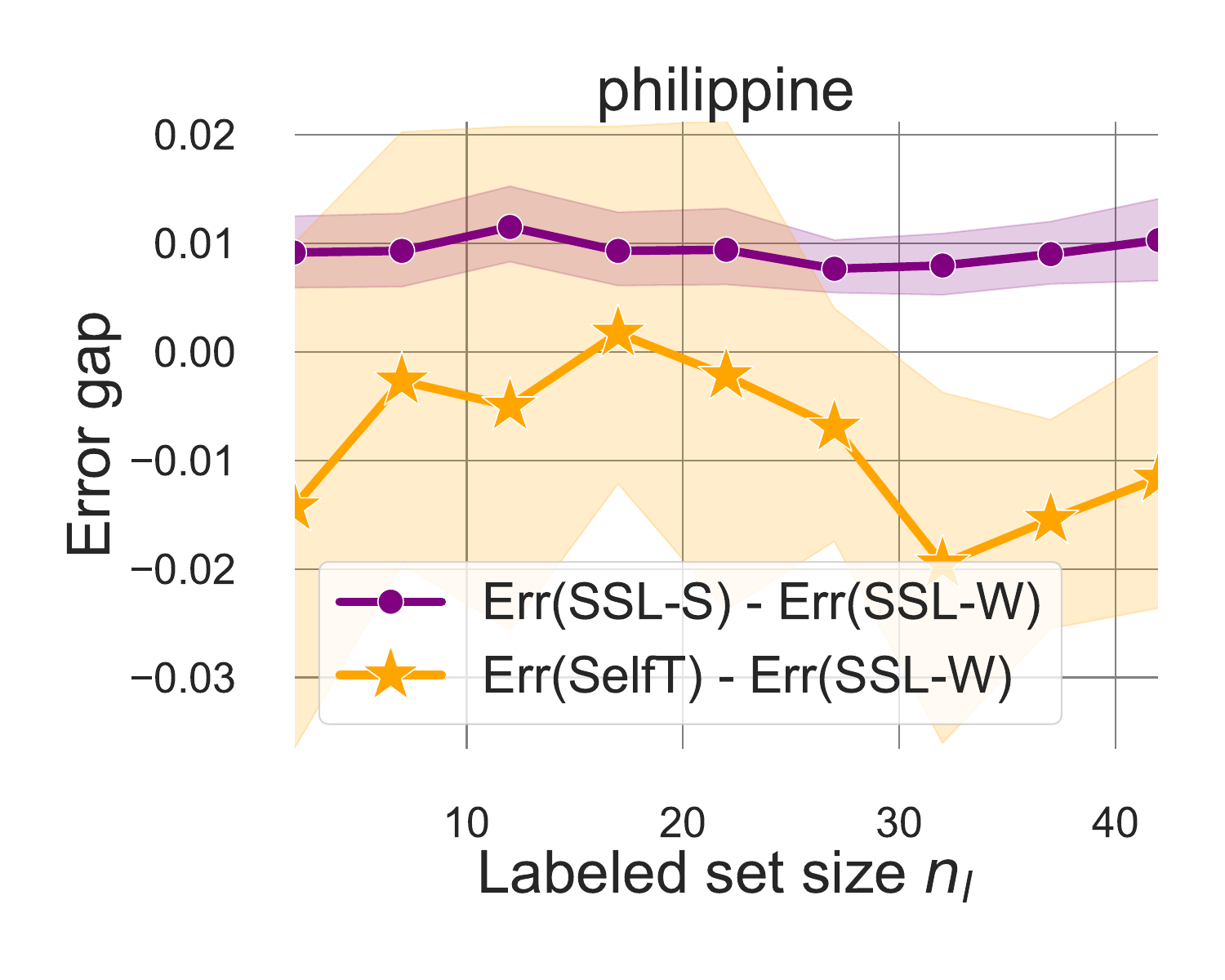}
  \end{subfigure}
  \begin{subfigure}[t]{0.32\textwidth}
    \centering
    \includegraphics[width=\textwidth]{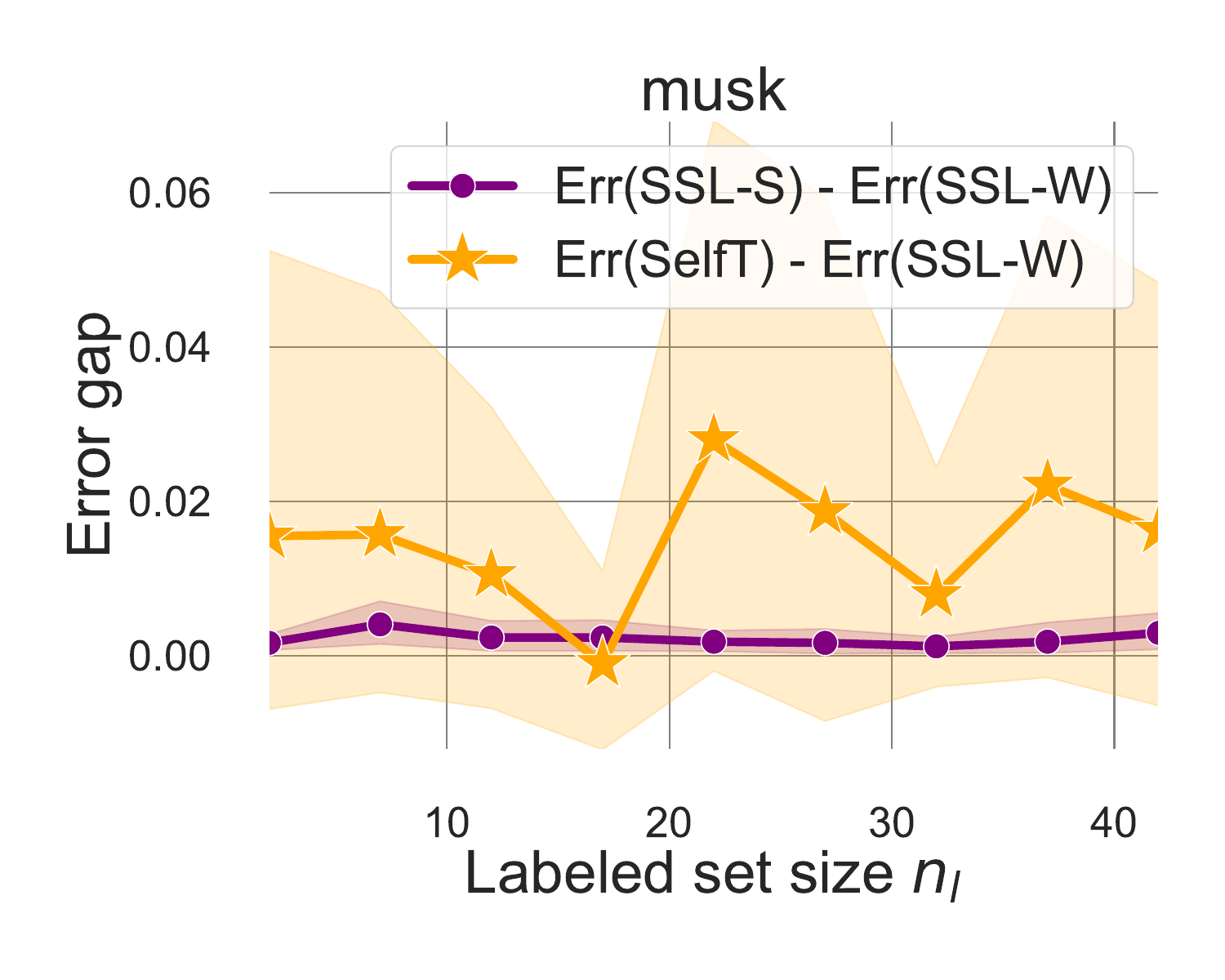}
  \end{subfigure}
  \begin{subfigure}[t]{0.32\textwidth}
    \centering
    \includegraphics[width=\textwidth]{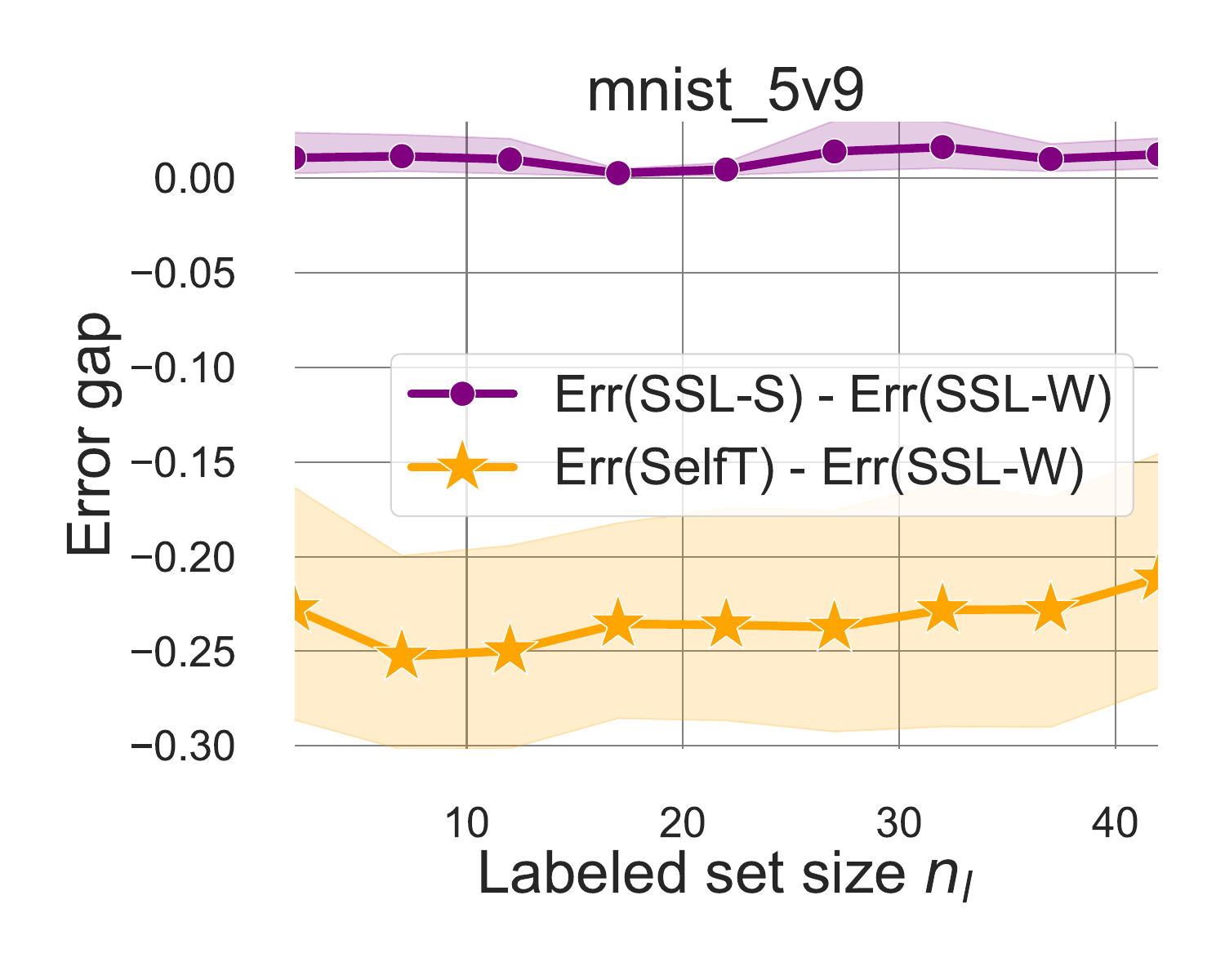}
  \end{subfigure}
  \vspace{-0.2cm}
  \begin{subfigure}[t]{0.32\textwidth}
    \centering
    \includegraphics[width=\textwidth]{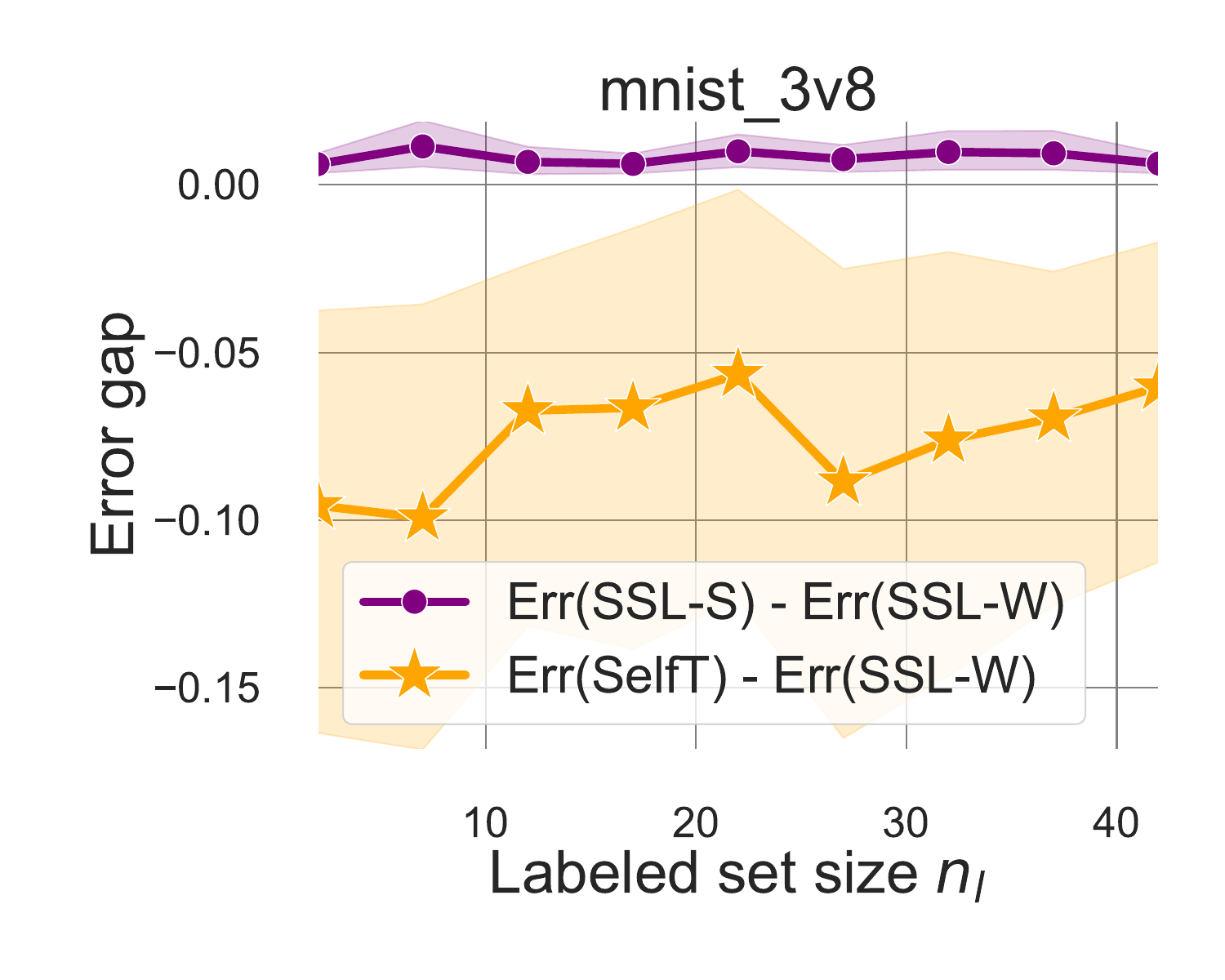}
  \end{subfigure}
  \begin{subfigure}[t]{0.32\textwidth}
    \centering
    \includegraphics[width=\textwidth]{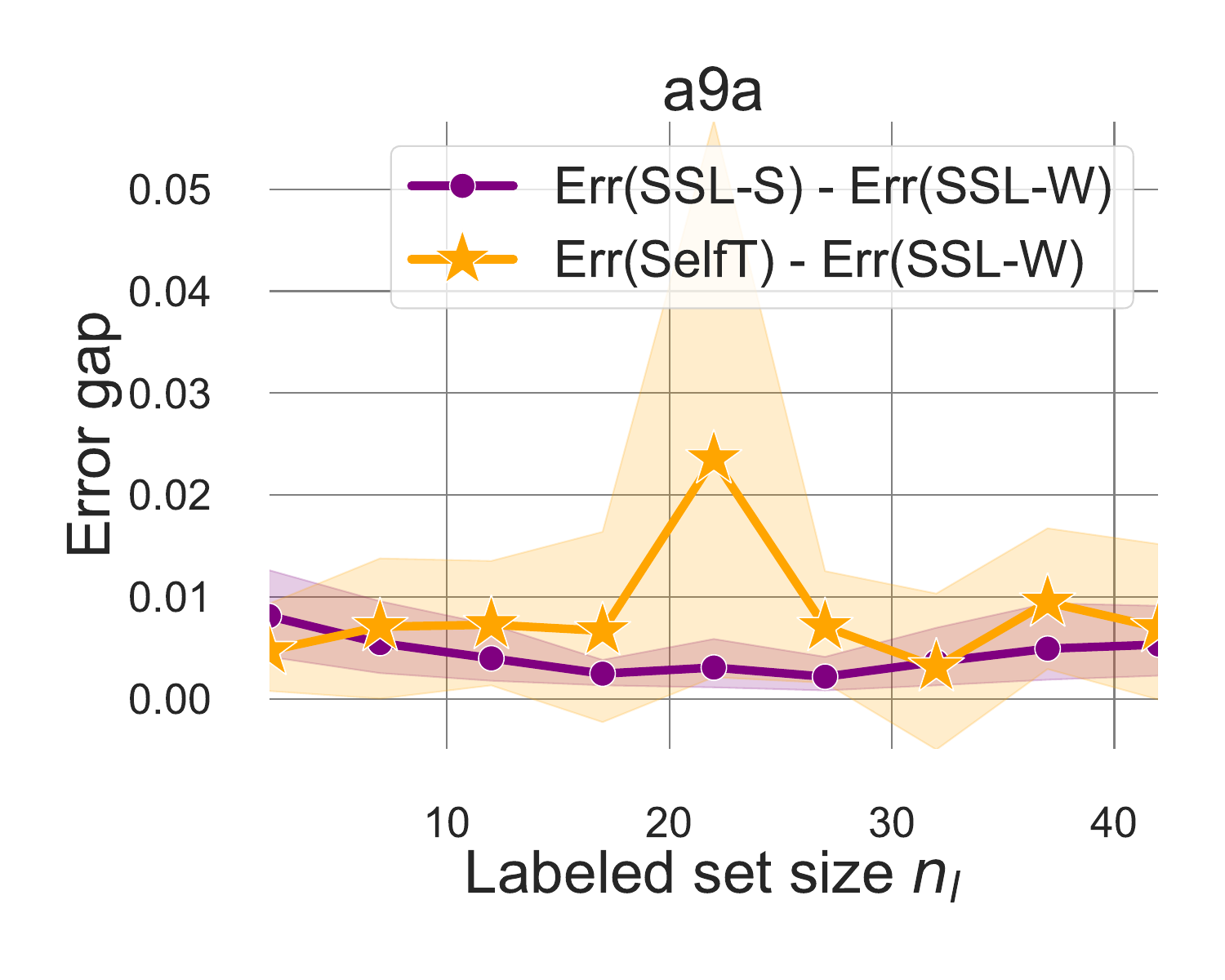}
  \end{subfigure}
  \begin{subfigure}[t]{0.32\textwidth}
    \centering
    \includegraphics[width=\textwidth]{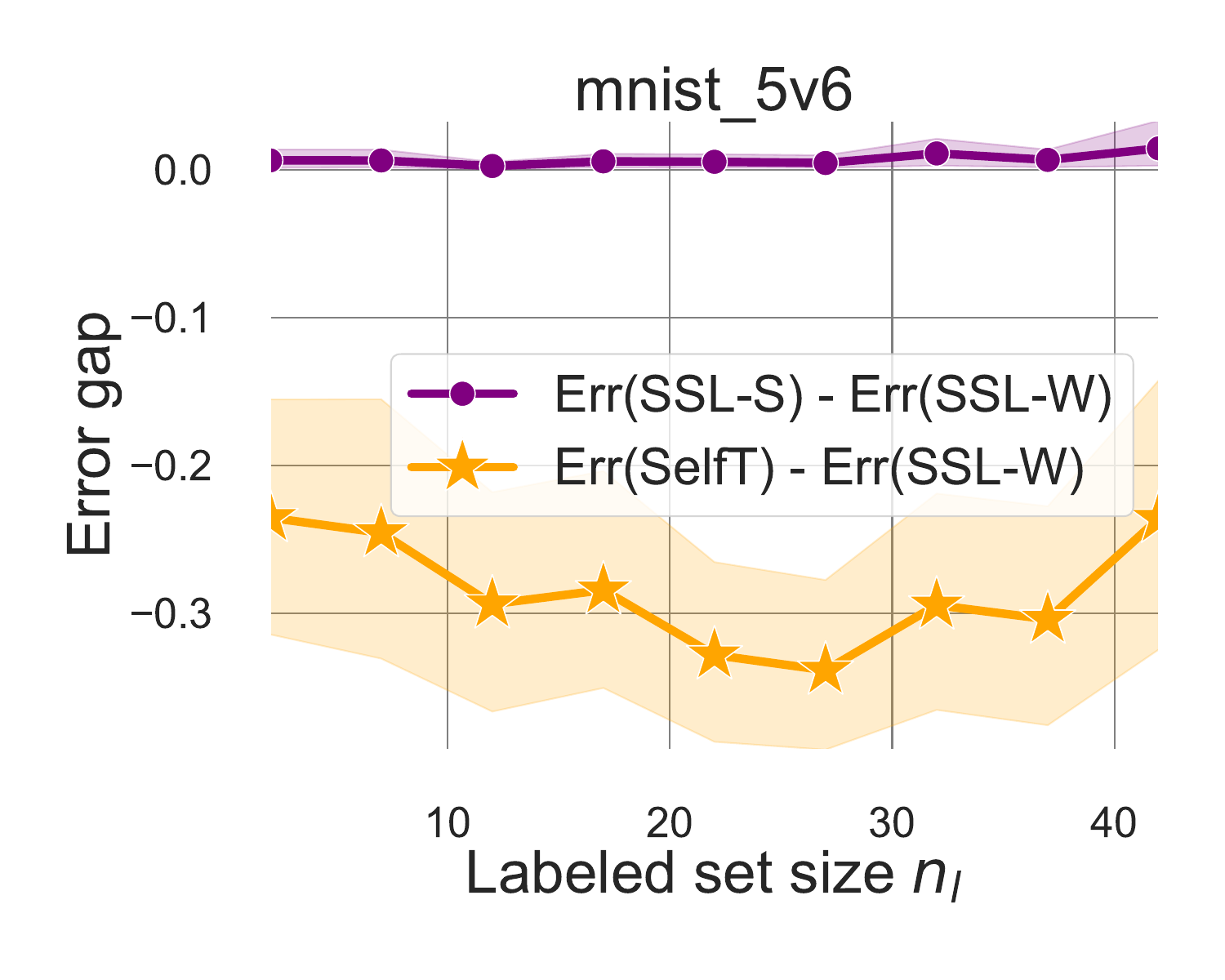}
  \end{subfigure}
  \vspace{-0.2cm}
  \begin{subfigure}[t]{0.32\textwidth}
    \centering
    \includegraphics[width=\textwidth]{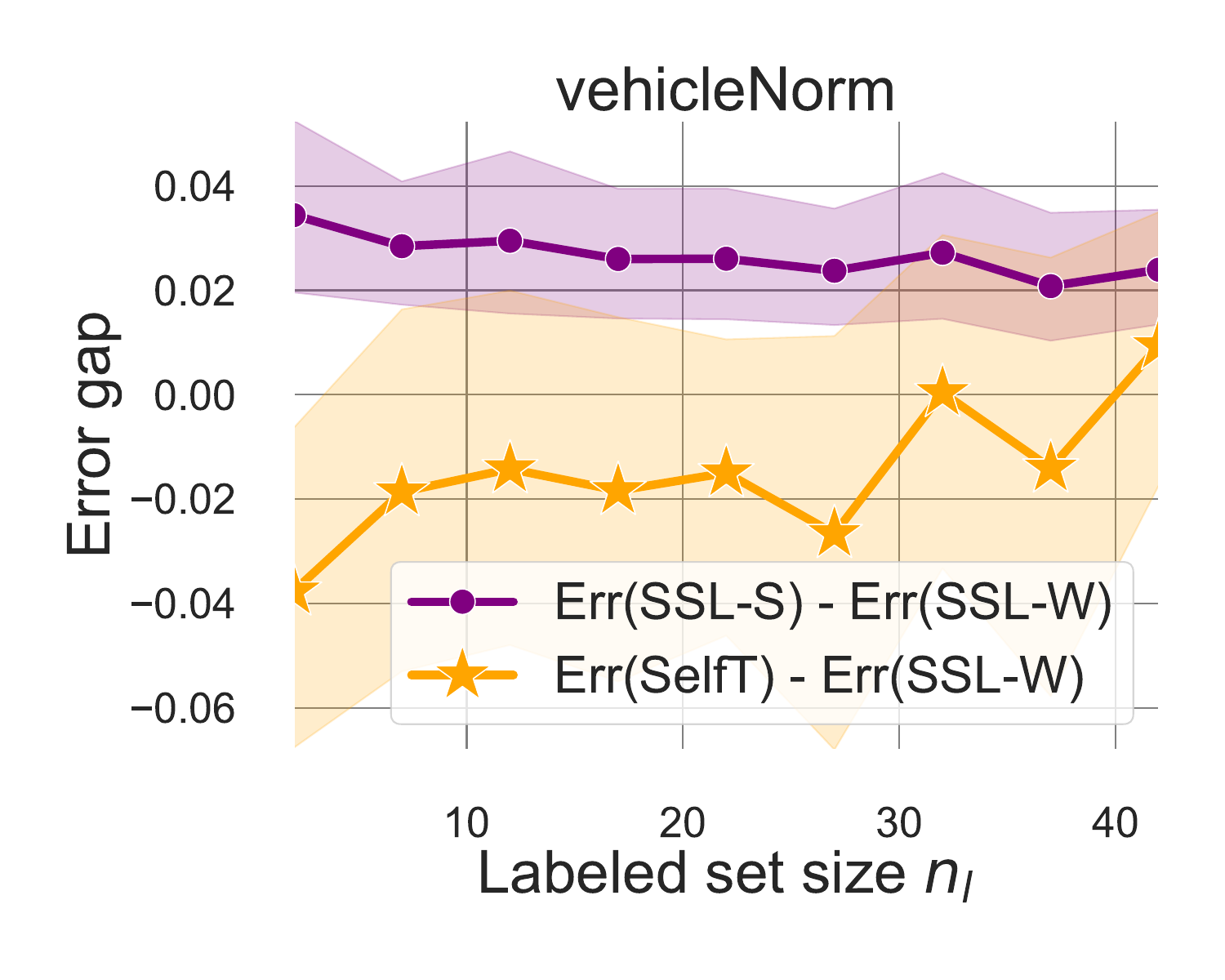}
  \end{subfigure}
  \begin{subfigure}[t]{0.32\textwidth}
    \centering
    \includegraphics[width=\textwidth]{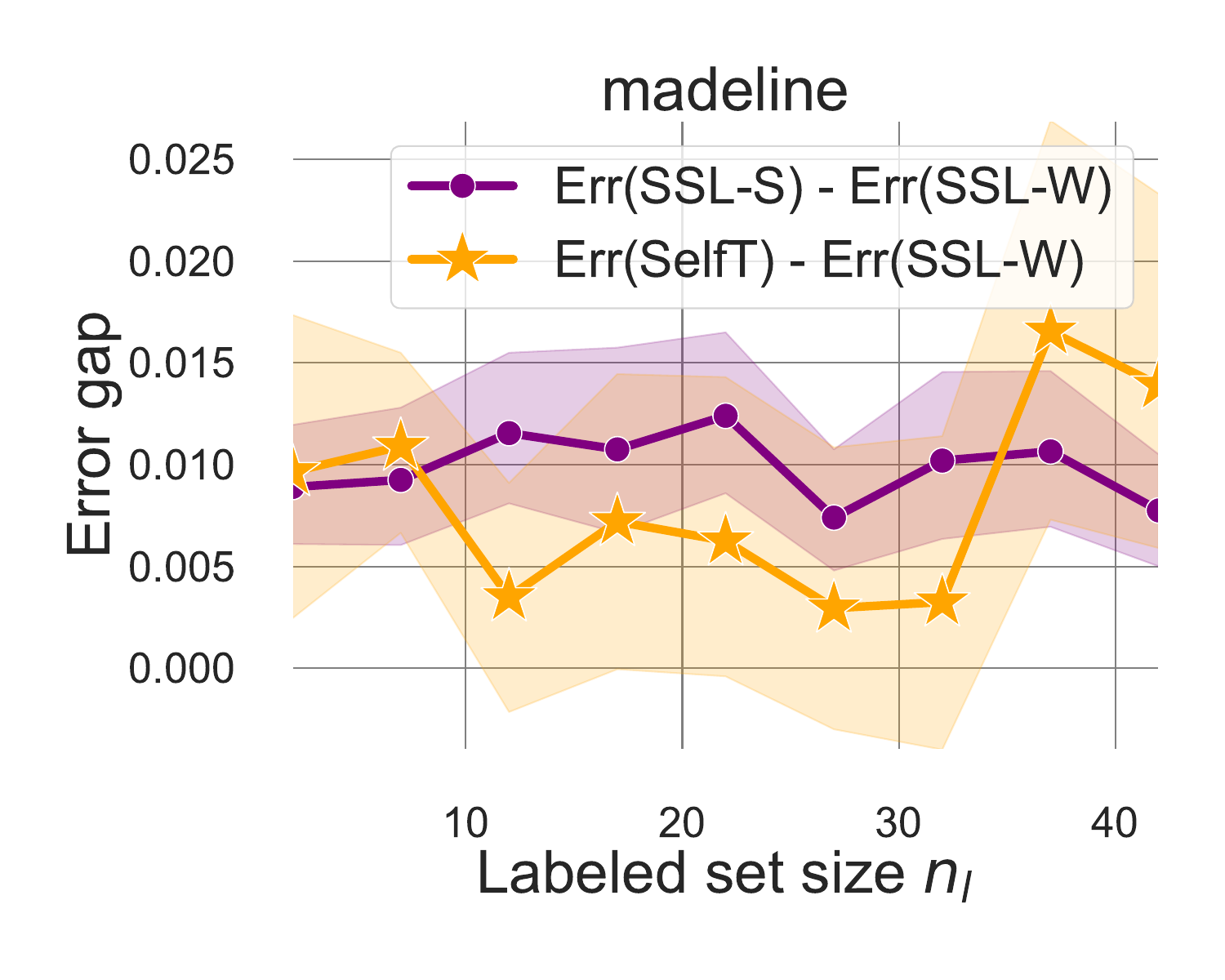}
  \end{subfigure}
  \begin{subfigure}[t]{0.32\textwidth}
    \centering
    \includegraphics[width=\textwidth]{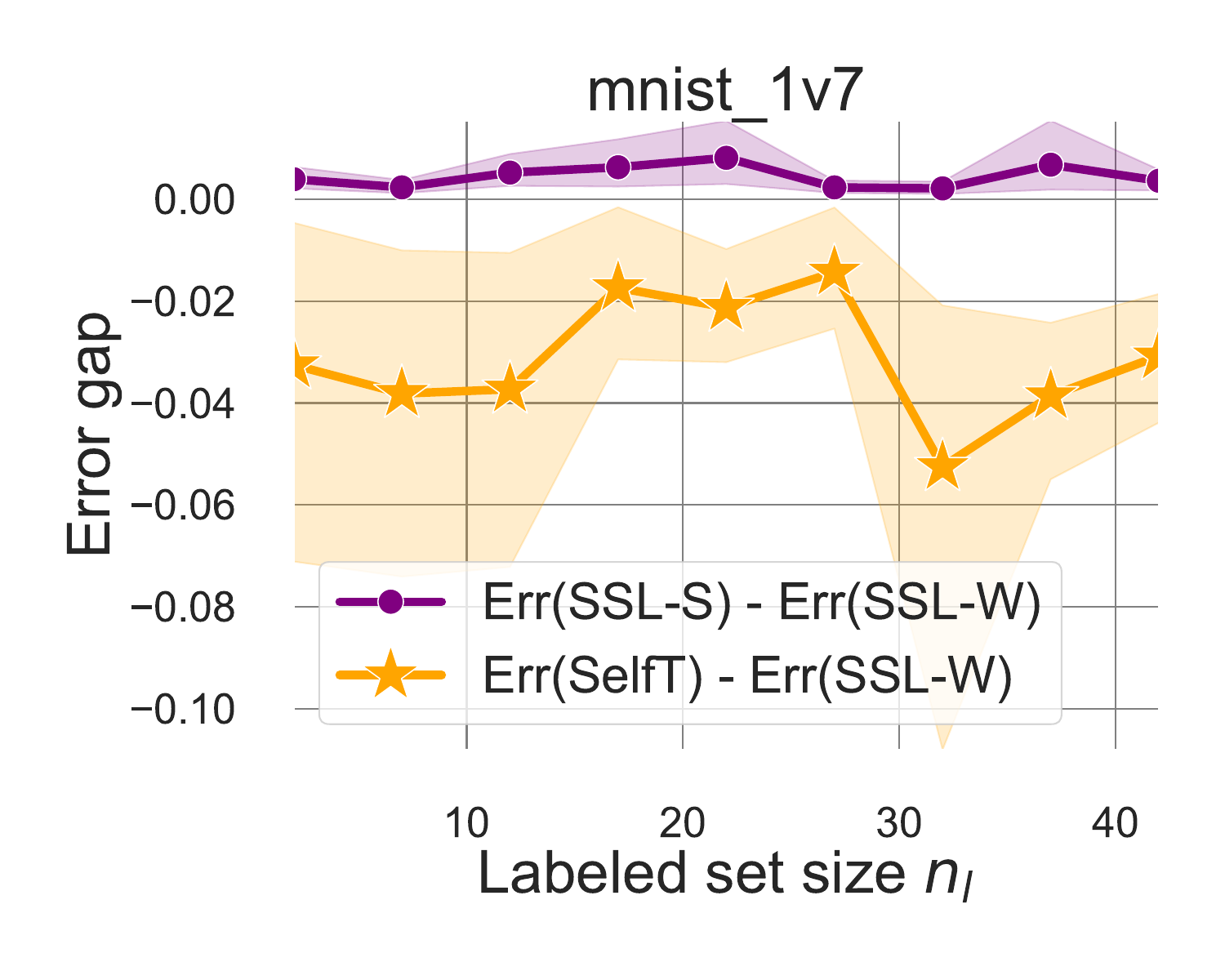}
  \end{subfigure}
  \vspace{-0.2cm}
  \begin{subfigure}[t]{0.32\textwidth}
    \centering
    \includegraphics[width=\textwidth]{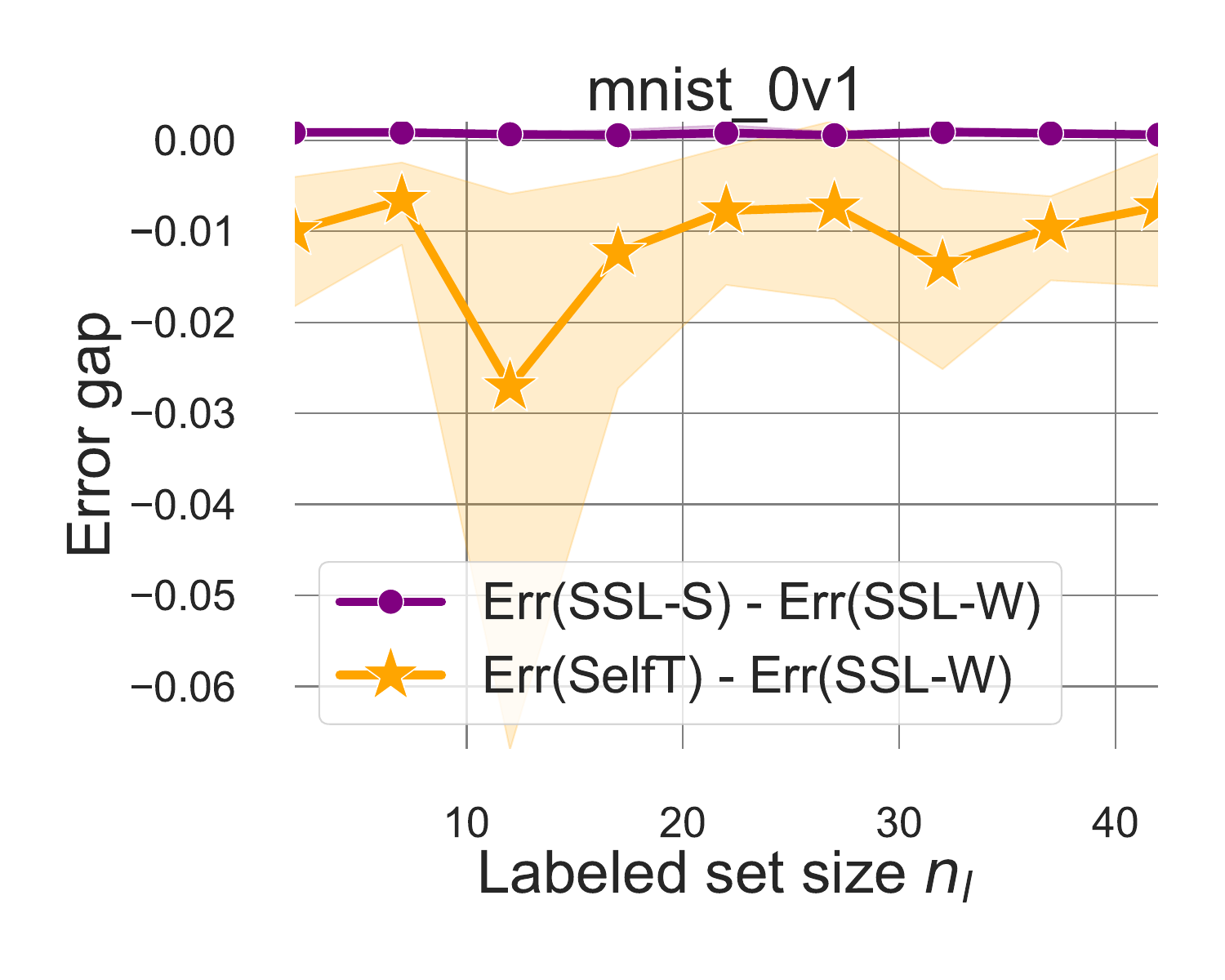}
  \end{subfigure}

  \vspace{-0.2cm}
  \caption{Error gap between SSL-S/self-training and SSL-W on real-world
  datasets. The positive gap indicates that SSL-W (and, in turn, self-training)
  outperforms SSL-S (and hence, also SL and UL+) for a broad range of $\nlab$ values. For all datasets, we set $\nunl=4000$.}
  \label{fig:more_switch_vs_weighted}
  \vspace{-0.5cm}

\end{figure*}


\end{document}